\documentclass{article} 
\usepackage{iclr2025_conference,times}
\iclrfinalcopy

\usepackage{amsmath,amsfonts,bm}

\def\1{\bm{1}}

\DeclareMathAlphabet{\mathsfit}{\encodingdefault}{\sfdefault}{m}{sl}
\SetMathAlphabet{\mathsfit}{bold}{\encodingdefault}{\sfdefault}{bx}{n}

\newcommand{\E}{\mathbb{E}}

\newcommand{\R}{\mathbb{R}}

\DeclareMathOperator*{\argmax}{arg\,max}
\DeclareMathOperator*{\argmin}{arg\,min}

\usepackage{colortbl}
\usepackage[utf8]{inputenc} 
\usepackage[T1]{fontenc}    
\usepackage{hyperref}       
\usepackage{url}            
\usepackage{booktabs}       
\usepackage{amsfonts}       
\usepackage{nicefrac}       
\usepackage{microtype}      
\usepackage{xcolor}         
\usepackage{multirow}       
\usepackage{paralist}       
\usepackage{mdwlist}        
\usepackage{graphicx}
\usepackage{color}          
\usepackage{xspace}                     
\usepackage{float}
\usepackage{amsmath}
\usepackage{amssymb}
\usepackage{amsthm}
\usepackage{enumitem}
\usepackage{algorithm}
\usepackage{algorithmic}
\usepackage{xcolor}	

\usepackage{mathtools}
\usepackage{tabularx}
\usepackage{xspace}
\usepackage{makecell}
\usepackage{wrapfig}
\usepackage{tabularx}
\usepackage{subcaption}
\definecolor{darkblue}{rgb}{0, 0, 0.5}
\hypersetup{colorlinks,linkcolor={darkblue},citecolor={darkblue},urlcolor={darkblue}}

\usepackage{setspace}

\usepackage{comment}

\usepackage{pifont}
\newcommand{\cmark}{\ding{51}}%
\newcommand{\xmark}{\ding{55}}%

\theoremstyle{plain}
\newtheorem{theorem}{Theorem}[section]

\newtheorem{lemma}[theorem]{Lemma}
\newtheorem{fact}[theorem]{Fact}

\theoremstyle{definition}
\newtheorem{definition}[theorem]{Definition}

\newtheorem{example}[theorem]{Example}
\theoremstyle{remark}
\newtheorem{remark}[theorem]{Remark}

\newcommand*\diff{\mathop{}\!\mathrm{d}}

\renewcommand{\P}{\mathop{\mathbb{P}}}

\newcommand{\watermark}{{\sf Watermark}}

\newcommand{\detect}{{\sf Detect}}

\newcommand{\wmk}{{\sf k}}

\newcommand{\fullres}[2]{$#1_{#2}$}

\def\cA{\mathcal{A}}

\def\cV{\mathcal{V}}

\newcommand{\logit}{{u}}

\definecolor{GoogleGreen}{RGB}{60,186,84}
\definecolor{GoogleRed}{RGB}{219,50,54}

\title{Permute-and-Flip: An optimally stable and \\watermarkable decoder for LLMs
}

\author{Xuandong Zhao \\ UC Berkeley \\ \small \texttt{xuandongzhao@berkeley.edu} 
\And
Lei Li \\ Carnegie Mellon University \\ \small \texttt{leili@cs.cmu.edu} 
\And
Yu-Xiang Wang \\ UC San Diego \\ \small \texttt{yuxiangw@ucsd.edu} 
}

\begin{document}

\maketitle

\begin{abstract}
In this paper, we propose a new decoding method called Permute-and-Flip (PF) decoder. It enjoys stability properties similar to the standard sampling decoder, but is provably up to 2x better in its quality-stability tradeoff than sampling and never worse than any other decoder. We also design a cryptographic watermarking scheme analogous to \citet{aaronson}'s Gumbel watermark, but naturally tailored for PF decoder. The watermarking scheme does not change the distribution to sample, while allowing arbitrarily low false positive rate and high recall whenever the generated text has high entropy. Our experiments show that the PF decoder (and its watermarked counterpart) significantly outperform(s) naive sampling (and its Gumbel watermarked counterpart) in terms of perplexity, while retaining the same stability (and detectability), hence making it a promising new approach for LLM decoding. The code is available at \url{https://github.com/XuandongZhao/pf-decoding}.

\end{abstract}


\section{Introduction}
\label{sec:intro}
Large language models (LLMs) \citep{OpenAI2022ChatGPT, OpenAI2023GPT4TR, bai2022constitutional,  touvron2023llama} have become increasingly popular in recent years due to their ability to generate human-like text and solve many tasks through a natural chatbot interface. 

A language model predicts the next word in a sentence using a real-value function $\logit(\cdot;\text{prompt},\text{prefix}): \cV\rightarrow \R$, known as \emph{logits}, which encodes the model's preferences on which word to choose. Here $\cV$ is the vocabulary space (typically a large discrete set of words); the ``prompt'' describes the task of interest; and ``prefix'' includes all preceding words that have been generated so far. 

A language model \emph{decoder} refers to a possibly randomized function that takes a prompt text $x$, API access to the \emph{logits} function as input, and outputs a sentence $y_{1:n}$. 

\textbf{The main thrust of this paper} is to introduce a new decoder, termed \emph{Permute-and-Flip decoding}, work out some of its intriguing properties with an application to watermarking LLM text, and hopefully convince readers that it deserves a shot at your next LLM application. 


\section{Problem Setup and Summary of Results}

Before diving in, let's set up the stage with a quick tour of existing decoding methods and have a brief look into 
how a language model decoder can be evaluated.

\textbf{Popular existing decoding methods} fall into three categories: (1) Planning-based methods such as beam search that aims at maximizing the sequence likelihood; (2) sampling-based methods that recursively sample from the next-word distribution, e.g., the soft(arg)max transform of the logits 
\begin{equation}\label{eq:sampling}
\textbf{Softmax sampling: }  \text{\small $y_t \sim p(y) = \frac{e^{\logit(y | x,y_{1:t-1}) / T}}{\sum_{\tilde{y}}e^{\logit(\tilde{y} | x,y_{1:t-1})/T}}$}
\end{equation}
where $T$ is the \emph{temperature} parameter; and (3) greedy methods such as greedy decoding that simply outputs  $y_t = \argmax_{y\in\cV} \logit(y | x,y_{1:t-1})$ as well as its Top $p$ \citep{holtzman2019curious} and Top $k$ sampling variants that interpolate greedy and sampling methods.


\textbf{Performance metrics.} How do we compare different decoding methods? More generally, how do we evaluate LLM-generated text? These are questions far from being settled.  Naturally, if there is a (possibly task-dependent) performance metric $U_x: \mathcal{V}^n \rightarrow \R$ one can define, then the optimal decoder would be the one that outputs $
y_{1:n}^* = \argmax_{y_{1:n}\in \cV^n} U_x(y_{1:n}).
$
Often $U_x$ is instantiated to be the sequence likelihood $\sum_{t=1}^n\log p( y_t | x, y_{1:t-1})$ which is equal to $\sum_{t=1}^n \logit_t(y_t)$. 


Recent works \citep{ippolito2019comparison,wiher2022decoding}, however, report that strategies that aim at maximizing sequence likelihood often result in texts that are more repetitive and less effective in some downstream tasks than those from the sampling-based methods \citep{holtzman2019curious}. 
Depending on what the task is, there is not a one-size-fits-all performance metric, therefore is no single decoding method that works well for all tasks.


For the moment, let us stash the disputes on how to best evaluate an LLM-generated text and focus on designing methods that maximize any user-specified utility function. In fact, we will also give up on solving the sequence-level utility maximization problem\footnote{It is known to be NP-Complete \citep{Chen2017RecurrentNN}.} and simply maximize a \emph{per-step} utility function $\logit_t: \cV \rightarrow \R$.

$\logit_t$ can simply be the logits function that LLMs output, which may have already accounted for potential future utility (like the $Q$ function in reinforcement learning) since the transformer-based language model had access to future texts during pre-training.  Or $\logit_t$ can be explicitly augmented with structure-inducing regularizers such as a lookahead heuristic as in A* decoding \citep{lu2021neurologic}, a retrieval-based term for fact-checking \citep{lewis2020retrieval}, or an entropy bonus for promoting diversity \citep{meister2020generalized}. 

\textbf{Our goal} is thus to construct a possibly randomized algorithm $\cA$ that takes $\logit_t$ as an input and outputs $y_t\in\cV$ that aims at maximizing $\E_{y_t\sim \cA_{u_t}}[\logit_t(y_t)]$ as much as possible. In the remainder of the paper, we will simply take $\logit_t$ as ``logits''  for a concrete exposition  --- all results are valid when $\logit_t$ is instantiated otherwise.



\textbf{Other constraints / consideration.} Why doesn't the trivial greedy decoder work?  
That's because there are other considerations besides text quality when selecting LLM decoders.
 For example, \textbf{computational efficiency and latency} are hugely important, since each API call to the  \emph{logits} function is costly. The \textbf{diversity} of the generated text is also important, especially for creative tasks.

Moreover, the decoding procedure should be \textbf{watermarkable} \citep{aaronson, kirchenbauer2023watermark,Zhao2023ProtectingLG,Kuditipudi2023RobustDW} in the sense that one should be able to inject subtle statistical signals that can be retrieved when given a secret key, to \emph{prove} that the text is generated by this particular language model.  Being watermarkable prevents the LLM from being used for malicious purposes such as scams \citep{Weidinger2021EthicalAS}, fake news \citep{zellers2019defending}, and plagiarism \citep{StokelWalker2022AIBC}.

In addition to the above, one may also hope the decoding algorithm to be \textbf{stable against small perturbations} to the \emph{logits}. Specifically, 
\begin{definition}[Stability]\label{def:stability}
    We say a decoding algorithm $\mathcal{A}$ is $L$-stable if for any prompt $x$, prefix $y_{\leq t}$, and for any perturbed $\tilde{\logit}$ such that $\|\tilde{\logit} - \logit\|_{\infty} \leq \delta$, the log-probability ratio satisfies 
    $$\left|\log\left\{\frac{p_{\mathcal{A}(\tilde{\logit}(\cdot|x,y_{\leq t}))}(y)}{p_{\mathcal{A}(\logit(\cdot|x,y_{\leq t}))}(y)}\right\}\right|\leq L\delta \;\;\forall y\in\cV.$$
\end{definition}
The stability helps to avoid catastrophic failure in the scenarios where the logits may be subject to data poisoning \citep{zhang2021inject, lin2021ml} or jailbreaking attacks \citep{Zhang2023OnTS, zhao2024weak}. Furthermore, \emph{stability} implies an intuitive notion of \emph{diversity}, suggesting that tokens with similar logits should have comparable probabilities of being selected. For further discussion and examples, please refer to Appendix \ref{sec:dis_stable}.

\begin{table}[t]
\centering
    \caption{Comparison of different decoding methods against five desiderata.}
    \label{tab:compare_decoders}
\resizebox{0.99\textwidth}{!}{
    \begin{tabular}{c|ccccc}
    \toprule
        Methods & Perplexity & Computational Efficiency & Diversity & Watermark & Stability \\ \midrule
        Search (e.g., Beam)      &Lowest & \xmark&\xmark&\xmark &\xmark \\
        Greedy      &Low & \cmark &\xmark&\xmark&\xmark \\
        Softmax Sampling  &Moderate   & \cmark&\cmark&\cmark&\cmark \\ 
        Top-$p$ Sampling     &Low (for small $p$)    & \cmark &Depends on $p$&\cmark&\xmark \\
        Top-$k$ Sampling     &Low (for small $k$)    & \cmark&Depends on $k$&\cmark&\xmark \\\midrule
        PF Sampling (ours)      &Lower than Softmax&\cmark&\cmark&\cmark&\cmark \\
    \bottomrule
    \end{tabular}
    }
\end{table}
Inspecting the decoding methods along the aforementioned dimensions, we notice that planning-based methods fail to be computationally efficient. 
While greedy decoding is efficient and has relatively low perplexity, its generated texts are neither diverse nor watermarkable (at least not using existing techniques). The sampling-based methods, however, are both watermarkable and diverse. In addition, softmax sampling is known to be $2$-stable, while the others that we have discussed are not stable.
\begin{fact}\label{fact:EM}
Softmax sampling decoding using \eqref{eq:sampling} with temperature $T$ satisfies $(2/T)$-stability.
\end{fact}
\begin{proof}
The result is implied by the differential privacy guarantee of exponential mechanism \citep[Theorem 6]{mcsherry2007mechanism}.
\end{proof}
The pros and cons of different decoding methods are summarized in Table~\ref{tab:compare_decoders}.
 From the table, we can see that there is a clear tradeoff between minimizing perplexity and preserving other properties. In particular, softmax sampling is the only method that checks all boxes, and the only one that is stable among existing decoders. This observation begs the following research question:
\emph{Is there a decoding method that is as stable as softmax sampling, but has lower perplexity?}
 



In this paper, we answer this question affirmatively by bringing in a technique called Permute-and-Flip sampling. 
Our contributions are fourfold.

\begin{enumerate}[leftmargin=*, itemsep=0pt, topsep=0pt]
    \item We introduce Permute-and-Flip decoding --- a new decoding algorithm for language models based on recent development in a very different context \citep{mckenna2020permute}.  
    \item We demonstrate that existing results from \citet{mckenna2020permute} already imply that:
    \begin{itemize}[leftmargin=*, itemsep=0pt, topsep=0pt]
        \item Permute-and-Flip decoding is provably stable. 
        \item The stability-perplexity tradeoff of the PF decoding is Pareto-optimal. In particular, when compared to softmax sampling, PF decoding has up to 2x smaller expected suboptimality while having the same stability parameter $L$.
    \end{itemize}
    \item We designed an analog of \citet{aaronson}'s Gumbel-Watermark for PF decoder, called the PF watermark. We show that the watermarked PF decoder samples from a distribution that is computationally indistinguishable from the non-watermarked PF decoder, and the detection procedure has precisely controlled false positive rate (FPR) and high power in identifying watermarked texts.
    \item We empirically demonstrate that on open-generation tasks, PF watermark achieves the best balance of the highest detection accuracy and lowest perplexity compared to the baselines.
\end{enumerate}

Overall, our proposed permute-and-flip decoding method provides a promising approach to balancing the tradeoff between perplexity and stability in LLM decoding while also admitting watermarking capabilities.

\paragraph{Related work and novelty.}  PF sampling was invented in the differential privacy (DP) literature \citep{mckenna2020permute}. Its stability properties are well-understood.  The stability of Softmax sampling is also well-known \citep{mcsherry2007mechanism}. Our contribution is in applying this method to LLM decoding and connecting these known theoretical results to the broader ML audience. To our knowledge, the PF watermark is new to this paper. The design of the watermark leverages the Report-Noisy-Max interpretation of the PF sampling \citep{ding2021permute} which allows a similar pseudo-random function like the work of \citet{aaronson} to be applied. A more thorough discussion of the related work is given in Appendix \ref{sec:related}.






\section{Permute-and-Flip Decoding its Properties}
\label{sec:pre}


The Permute-and-Flip decoding iteratively generates the next token by a simple procedure that uses only the logits.
It involves first randomly permuting the vocabulary, then flipping a sequence of biased coins according to the permuted sequence until the first ``head'' is seen (see Algorithm~\ref{alg:PF_decoding}).

\begin{algorithm}[H]
\small
   \caption{Permute and Flip (PF) Decoding}\label{alg:PF_decoding}
\begin{algorithmic}[1]
   \STATE {\bfseries Input:} prompt $x$, language model $\mathcal{M}$, temperature $T$.
   \FOR{$t=1, 2, \cdots$}
   \STATE Logits $\logit_t\leftarrow \mathcal{M}([x, y_{1:t-1}])$. 
   \STATE Find $\logit_t^*\leftarrow \max_{y\in\mathcal{V}} \logit_t(y)$.
   \STATE $\mathsf{Permute:}$ Shuffle the vocabulary $\cV$ into $\tilde{\cV}$.
    \FOR{$y\in \tilde{\cV}$}
       \STATE $\mathsf{Flip:}$ Draw $Z\sim \textrm{Bernoulli}\left(\exp{(\frac{\logit_t(y)-\logit_t^*}{T})}\right)$. 
       \STATE \textbf{if} $Z = 1$,  \textbf{then} assign $y_t \leftarrow y$ and \textbf{break}.
    \ENDFOR
   \ENDFOR
   \STATE {\bfseries Output:} Generated sequence $y = [y_1,...,y_n]$. 
\end{algorithmic}
\end{algorithm}
Permute-and-flip makes words with higher logits exponentially more likely --- even more so than Softmax sampling (Eq. \ref{eq:sampling}). To see this, one may consider a rejection sampling algorithm for obtaining a sample from Eq. (\ref{eq:sampling}), which repeats the following procedures until it halts.
\begin{enumerate}[leftmargin=*, itemsep=0pt, topsep=0pt]
    \item Uniformly samples $y\in \mathcal{V}$,
    \item Return it with probability: $$p(y)/ p(y^*) = \exp\left( (\logit_t(y) - \logit_t(y^*))/T\right).$$
\end{enumerate}

This procedure differs from PF sampling in that it samples $y$ \emph{with replacement}, whereas PF sampling samples $y$ \emph{without replacement}, giving PF sampling a higher likelihood of producing $y^*$.

PF sampling was initially proposed in \citet{mckenna2020permute} as a differentially private selection mechanism that has better utility than the more well-known exponential mechanism \citep{mcsherry2007mechanism}.  \citet{mckenna2020permute} also derived a plethora of theoretical properties of the PF sampling. The following theorem summarizes these results in the language of LLM decoding.
\begin{theorem}\label{thm:pf_sampling}
Let the logits function be $\logit$ and $\logit^* = \max_{y\in\cV} \logit(y)$. Let $\mathrm{PF}(\logit)$ be the distribution of PF-sampling, and $\mathrm{Softmax}(\logit)$ be the distribution in \eqref{eq:sampling}, both with temperature parameter $T$. The following statements are true.
\begin{enumerate}[leftmargin=*, itemsep=0pt, topsep=0pt]
     \item (\textbf{Same stability}) PF-Sampling is $(2/T)$-stable.
     \item  (\textbf{Nearly greedy})  PF-sampling obeys, $$\E_{y\sim\mathrm{PF}(\logit)}[\logit(y)  ]\geq  \logit^* - T\log|\cV|.$$ 
    \item (\textbf{``Never worse''}) For the same $T$, PF-sampling is never worse than Softmax-sampling.
    $$
            \E_{y\sim\mathrm{PF}(\logit)}[\logit(y)]\geq \E_{ y\sim\mathrm{Softmax}(\logit)}[\logit(y)]
    $$
    \item (\textbf{``Up to 2x better''}) There exists logits $\logit$ such that PF-sampling is 2x smaller in terms of suboptimality,
    $$
    \E_{y\sim\mathrm{PF}(\logit)}[ \logit^* - \logit(y)  ]\leq \frac{1}{2} \E_{y\sim\mathrm{Softmax}(\logit)}[ \logit^* - \logit(y)  ].
    $$
    \item 
    (\textbf{Optimal stability-perplexity tradeoff}) For any decoder $P$  that is $2/T$-stable, if there exists $\logit$ such that 
    $$\E_{y\sim P(\logit)}[  \logit(y)  ]> \E_{y\sim\mathrm{PF}(\logit)}[ \logit(y)  ] 
    $$
    then there must be another $\tilde{\logit}$ such that 
   $$\E_{y\sim P(\tilde{\logit})}[  \tilde{\logit}(y)  ]<  \E_{y\sim\mathrm{PF}(\tilde{\logit})}[\tilde{\logit}(y)  ]. 
    $$
\end{enumerate}
\end{theorem}
\begin{proof}
The theorem follows directly from \citet{mckenna2020permute}, specifically Theorem 1, Corollary 1, Theorem 2, Proposition 4, and Proposition 6.
\end{proof}

\begin{wrapfigure}{r}{0.45\textwidth}
\centering
\includegraphics[width=0.204\textwidth]{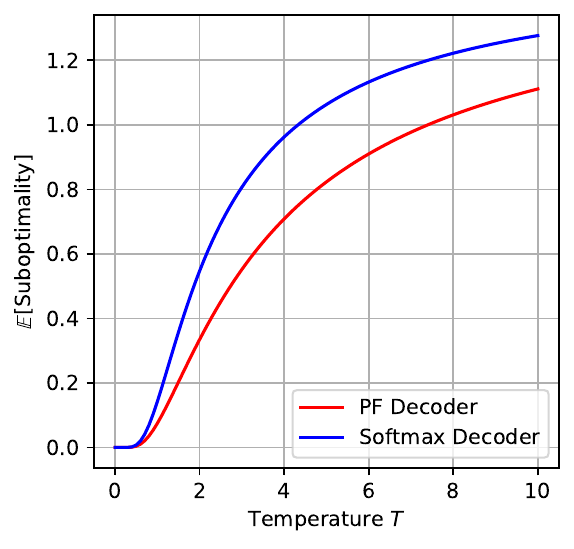}
\includegraphics[width=0.209\textwidth]{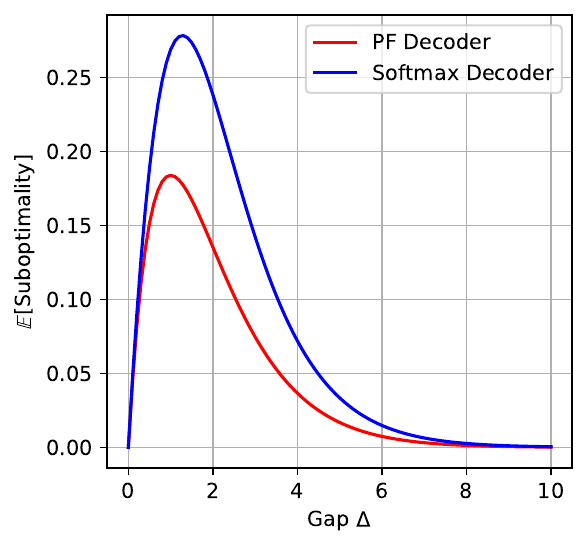}
\caption{Comparing PF decoder vs Softmax decoder using Example~\ref{ex:twotoken}. }
\vspace{-1.em}
\label{fig:suboptimality_fixed_temprature}
\end{wrapfigure}
The first statement shows that the PF decoder enjoys exactly the same stability parameter as in Fact~\ref{fact:EM}.  The second statement provides a worst-case bound on how far PF-sampling is away from greedy-decoding as a function of the temperature $T$ in terms of the likelihood achieved. The third and fourth statements show that PF-sampling is always ``more greedy'' than softmax-sampling. The last statement shows that PF-sampling is not dominated by any other decoder that is equally stable (as in Definition~\ref{def:stability}), thus \emph{Pareto optimal}. These results provide strong justification on the superiority of the permute-and-flip decoder over the standard softmax sampling in minimizing perplexity.

Let's consider a simple example to compare PF decoder and Softmax decoder.
\begin{example}\label{ex:twotoken}
     Let the $|\cV|=2$ and the corresponding logits be $[\Delta, 0]$ for gap $\Delta>0$. Softmax decoder chooses the suboptimal token with probability $1/(1+e^{\Delta/T})$, while PF decoder chooses it w.p. $1/({2}e^{\Delta/T})$.
\end{example}
Since $1/(1+x)> 1/(2x)$ for all $x>1$, the probability that the suboptimal token is chosen in PF sampling is strictly smaller than that of Softmax sampling. As shown in Figure~\ref{fig:suboptimality_fixed_temprature}, on the left, we fix the Gap $\Delta=3.0$ and vary the temperature $T$. On the right, we fix $T=1.0$ and vary $\Delta$. PF beats Softmax in all cases.




\section{Report-Noisy-Max and Watermarking}
\label{sec:method}

Next, we turn to the well-motivated problem of watermarking LLM generated text. The watermarking problem aims at embedding a secret message in the generated text that (essentially) reads ``Beware! I am written by an AI!''. 

The hope is that this message can be seen by anyone who has access to a \emph{secret key}, while ensuring that the watermarked version of the LLM generates text that has \emph{almost the same} distribution as (or at least very similar) to the original LLM. 

More formally, a watermarking scheme includes a ``$\watermark$'' function that injects the watermark and a ``$\detect$'' function that takes a suspect text sequence $y_{1:n}$ as input and outputs a prediction of $1$ (``It is watermarked!'') or $0$ (``It is not!''). 

A wrong accusation of non-watermarked text as watermarked is called a \emph{false positive}. A failure to detect a watermarked text is called a \emph{false negative}.
The performance of a watermark is measured by its detection \emph{power} (i.e., $1 - $\emph{false negative rate})  at a given \emph{false positive rate}. 

There are many other necessary properties for a watermarking scheme to be useful, such as \emph{low-overhead}, \emph{model-agnostic} detection, and \emph{resilience to edits} and other evasion attacks. We refer readers to the slide deck of \citet{aaronson} and the related work section of \citep{zhao2023provable} for a review of these desiderata and known results. 

Among the recent attempts, two popular watermarking schemes perform satisfactorily on all the above criteria. 
\setlist{nolistsep}
\begin{description}
\itemsep0em
    \item [Gumbel Watermark \citep{aaronson}] that uses a ``traceable'' pseudo-random softmax sampling when generating the next word.
    \item [Green-Red Watermark \citep{kirchenbauer2023watermark}] that randomly splits the vocabulary into Green and Red then slightly increases the logits for green tokens.
\end{description}
Both of them determine their pseudo-random seeds chosen according to the $m$ preceding tokens of the current token being generated. We will focus on explaining the Gumbel watermark as it is more closely related to our approach.

\paragraph{\citet{aaronson}'s Gumbel watermark.}
The key idea of the Gumbel watermark leverages the ``Gumbel-Max Trick'', which states that:
\begin{fact}[{\citealp{gumbel1948statistical}}]
 The softmax sampling in \eqref{eq:sampling} is equivalent to the following procedure 
\begin{equation}
    y_t =  \argmax_{y\in\cV}  \frac{\logit_t(y)}{T} +  G_t(y)
\end{equation}
where $G_t(y)\sim\textrm{Gumbel}(0,1)$ i.i.d for each $t,y$. 
\end{fact}
Gumbel noise can be generated using a uniform r.v.
$$
\textrm{Gumbel}(0,1)\sim -\log\left(\log(1/\mathrm{Uniform}([0,1]))\right).
$$ So given a random vector $r_{t}\sim (\mathrm{Uniform}([0,1]))^{|\cV|}$, we can write $G_t(y)=-\log(-\log(r_t(y)))$.

The $\watermark$ stage for the Gumbel-watermark essentially replaces $\mathrm{Uniform}([0,1])$ with a \emph{pseudo-random function} $r_t(y) = F_{y_{t-m:t-1},\wmk}(y)$. Given the secret key $\wmk$, the \emph{pseudo-random function} is a deterministic function with range $[0,1]^{\cV}$, but over the distribution of the secret key $\wmk$, $r_t$ is computationally indistinguishable from sampled from truly i.i.d. uniform distribution, which ensures that the distribution of $y_t$ in the watermarked model is computationally indistinguishable to the unwatermarked distribution \eqref{eq:sampling}.  

At $\detect$ phase of the  the Gumbel watermark, the auditor who has access to the key $\wmk$ may compute
$$
\mathrm{TestScore}_{\mathrm{Gumbel}}(y_{1:n})  = \sum_{t=m+1}^n  -\log(1-r_{t}(y_t)).
$$
If $y_{1:n}$ is \emph{not} generated from the watermarked model, then the test statistic is a sum of exponential random variable thus $\E[\mathrm{TestScore}(y_{1:n})] = n - m.$  
Meanwhile, if $y_{1:n}$ is generated by the Gumbel watermarked model, 
\begin{align}
&\E[\mathrm{TestScore}(y_{1:n})] = \sum_{t=m+1}^n\E \left[ \sum_{y\in \cV} p_t(y) H_{ \frac{1}{p_t(y)}}\right] \label{eq:aaronson_ts_mean} \\
&\geq (n-m) + \left(\frac{\pi^2}{6}-1\right) \sum_{t=m+1}^n\E\left[\textrm{Entropy}[p_t(\cdot)]\right]. \label{eq:aarons_ts_entropy}
\end{align}
where $p_t:=\mathrm{Softmax}(\logit_t/T)$, $H_\alpha := \int_{0}^1 \frac{1-x^\alpha}{1-x} dx$ is Euler's Harmonic number and $\textrm{Entropy}$ denotes the standard Shannon entropy (in nats) for a discrete distribution, i.e., $\textrm{Entropy}[p] = -\sum_{y\in\cV} p(y) \log p(y)$.  The above results were described in \citet{aaronson}'s talk but no detailed proofs were published to the best of our knowledge. To be self-contained, we added formal statements and proofs about the Gumbel watermarks in Appendix~\ref{app:gumbelwm} (see Theorem~\ref{thm:gumbel_test_score}~and~\ref{thm:aaronson_entropy}).


\paragraph{Permute-and-Flip as ReportNoisyMax.} It turns out that the Permute-and-Flip sampling has a similar equivalent Report-Noisy-Max form.  Instead of Gumbel noise, it is the exponential noise that are added to the logits. This less-known fact is due to \citet{ding2021permute}
\begin{fact}[{\citealp[Theorem~5]{ding2021permute}}]\label{fact:PF}
Permute-and-Flip Sampling in Algorithm~\ref{alg:PF_decoding} with parameter $T$ is equivalent to 
\begin{equation}
    y_t =  \argmax_{y\in\cV}  \frac{\logit_t(y)}{T} +  E_t(y).
\end{equation}
where $E_t(y) \sim \textrm{Exponential}(1)$ i.i.d. for each $t,y$.
\end{fact}
Leveraging this fact, in the remainder of the section, we develop a watermarking scheme for ReportNoisyMax that is analogous to the Gumbel-watermark.

\paragraph{Permute-and-Flip watermark.}
The natural idea is to replace the exponential noise $E_t(y)$ with a pseudo-random version that depends on a secret key and a prefix with length $m$.  Observe that $\mathrm{Exponential}(1)\sim -\log(\mathrm{Uniform}([0,1]))$, thus the standard pseudo-random function that generates uniform random variables can be used.  In the detection phase, we compute:
$$
\mathrm{TestScore}_{\mathrm{PF}}(y_{1:n}) = \sum_{t=m+1}^n  -\log(r_{t}(y_t)).
$$
Note that this is a simple change of sign of $r_t(y_t)$ comparing to the test score of the Gumbel watermark.  Detailed pseudo-code for how the watermark works are given in Algorithm~\ref{alg:wm} and Algorithm~\ref{alg:detect}.


\begin{algorithm}[t]
\small
   \caption{PF watermarking: $\watermark$}\label{alg:wm}
\begin{algorithmic}[1]
    \STATE {\bfseries Preparation:}   Randomly sample a watermark key $\wmk$ 
   \STATE {\bfseries Input:} Prompt $x$, language model $\mathcal{M}$, pseudo-random function $F$, watermark key $\wmk$, temperature $T$
   \FOR{$t=1, 2, \cdots$}
   \STATE Compute logits: $\logit_t\leftarrow \mathcal{M}([x, y_{1:t-1}])$
   \STATE Generate a pseudo-random vector $r_t(\cdot)$ using
   $r_t(y) := F_{y_{t-m:t-1},\wmk}(y)$ for $y\in \cV$
   \STATE Select the next token $y_t$ using 
   \begin{equation}\label{eq:wm}
    y_t = \argmax_{y \in \mathcal{V}} \left( \frac{\logit_t(y)}{T} - \log {r_t(y)} \right)
    \end{equation}
   \ENDFOR
   \STATE {\bfseries Output:} Watermarked sequence $y = [y_1,...,y_n]$
\end{algorithmic}
\end{algorithm}
\begin{algorithm}[t]
\small
   \caption{PF watermarking: $\detect$}\label{alg:detect}
\begin{algorithmic}[1]
   \STATE {\bfseries Input:} Suspect text $y_{1:n}$, watermark key $\wmk$, pseudo-random function $F$, target false positive rate $\alpha$
   \STATE {\bfseries Output:} Binary decision (1 if text is watermarked, 0 otherwise)
   \STATE Calculate the cumulative score
   \begin{equation}\label{eq:z}
\mathrm{TestScore}_{\mathrm{PF}}(y_{1:n}) = \sum_{t=m+1}^n  -\log(r_{t}(y_t))
   \end{equation}
   where $r_t(y) = F_{y_{t-m:t-1},\wmk}(y)$
   \STATE {\bfseries if} $\mathrm{TestScore} > \mathrm{CDF}_{\mathrm{Gamma}(n-m, 1)}^{-1}(1-\alpha)$ {\bfseries then return}  1, i.e.,  ``The suspect text is watermarked.'' 
   \STATE {\bfseries else return} 0, i.e., ``The suspect text is not watermarked.'' 
\end{algorithmic}
\end{algorithm}

\begin{theorem}\label{thm:PF_watermark}
Assume the pseudo-randomness is perfect\footnote{This is a simplifying assumption. 
We only need $(n-m)|\cV|$-way independence.}, i.e., $F_{w_{1:m},\wmk}(y)\sim \mathrm{Uniform}([0,1])$ i.i.d. $\forall [w_{1:m},y]\in\cV^{m+1}$.

The following are true about PF watermark scheme.
\begin{enumerate}[leftmargin=*, itemsep=0pt, topsep=0pt]
    \item If $y_{1:n}$ is \emph{statistically independent} to the secret key $\wmk$, 
$$\E\left[\mathrm{TestScore}_{\mathrm{PF}}(y_{1:n}) \middle| y_{1:n}\right]  = n-m. $$
    \item If in addition, all $m$-grams in $y_{1:n}$ are unique, then conditioning on $y_{1:n}$,  
    $$\mathrm{TestScore}_{\mathrm{PF}}(y_{1:n}) \sim \mathrm{Gamma}(n-m, 1).$$
The choice $\tau= \mathrm{CDF}_{\mathrm{Gamma}(n-m, 1)}^{-1}(1-\alpha)$ ensures the false positive rate in Algorithm~\ref{alg:detect} is equal to $\alpha$.
    \item Assume $y_{1:n}$ is drawn from Algorithm~\ref{alg:wm}, then
    \begin{equation}
    \label{eq:PF_mean_under_H1}
    \begin{aligned}
    \scriptsize
    \text{$\E\left[\mathrm{TestScore}_{\mathrm{PF}}(y_{1:n})\right] 
    = \sum_{t=m+1}^n \E\left[\sum_{y\in\cV}\int_{0}^{e^{\logit_t(y) - \logit_t^*}} \bigg(-\log r  \cdot \prod_{y'\in \cV, y'\neq y}\left( 1 - r\cdot e^{\logit_t(y') - \logit_t(y)}\right) \bigg) d r \right]$.}
    \end{aligned}
    \end{equation}
\end{enumerate}
\end{theorem}

The above expression in \eqref{eq:PF_mean_under_H1} may appear messy, but it is the exact calculation and captures the entropy of the distribution PF-induces for a given $\logit_t$. To see this, let us instantiate the bound for two special cases that admit more explicit forms. 
\begin{example}\label{ex:uniform_k_subset}
 When $\mathrm{Softmax}(\logit_t)$ is $1/k$ for an arbitrary subset of $k$ tokens and $0$ for others, 
$$
\E[-\log(r_{t}(y_t))] :=  H_{k} =  1 + 1/2 +...+1/k \approx \log k.
$$
Specifically, when  $k = |\cV|$ this is the uniform distribution, \eqref{eq:PF_mean_under_H1} $\asymp n\log|\cV|$ while when $k=1$, the sequence is completely deterministic (e.g., when the LLM is asked to recite the ``Declaration of Independence''), then we get \eqref{eq:PF_mean_under_H1} $= n-m$ as expected.
\end{example}
In the above example, \eqref{eq:PF_mean_under_H1} is identical to the expected TestScore of the Gumbel watermark in \eqref{eq:aaronson_ts_mean}. 
This is because the distributions they sample from are also the same. To illustrate their difference, let us revisit the simple two-token case from Example~\ref{ex:twotoken} again for which we can work out the expectation of the test score explicitly. 
\begin{example}\label{ex:twotoken_watermark}
     Let the $|\cV|=2$ and the corresponding logits be $[\Delta, 0]$. The expected TestScore of the Gumbel and the PF watermark (for each watermarked token) are:
     $\frac{H_{1+e^{-\Delta/T}}}{1+e^{-\Delta/T}} + \frac{H_{1+e^{\Delta/T}}}{1+e^{\Delta/T}}$
     and $1+\frac{1}{2}e^{-\Delta/T}(1+\Delta/T)$ respectively, where $H_x$ is the $x^{th}$ Harmonic number.
\end{example}
It is a bit hard to compare them by reading the mathematical expressions, so let us compare them numerically 
(see Figure~\ref{fig:detectability_fixed_stability}). The vertical axis in the figures measures \emph{Detectability}, which we define to be the expected difference between the TestScore of a watermarked and unwatermarked token. Since under the null the $\E[-\log(r_t(y_t))]=E[-\log(1-r_t(y_t))] = 1$, we can simply subtract $1$ from the expressions in Example~\ref{ex:twotoken_watermark}.

\begin{figure}[t]
    \centering
\begin{subfigure}[t]{0.47\textwidth}
\includegraphics[width=0.49\linewidth]{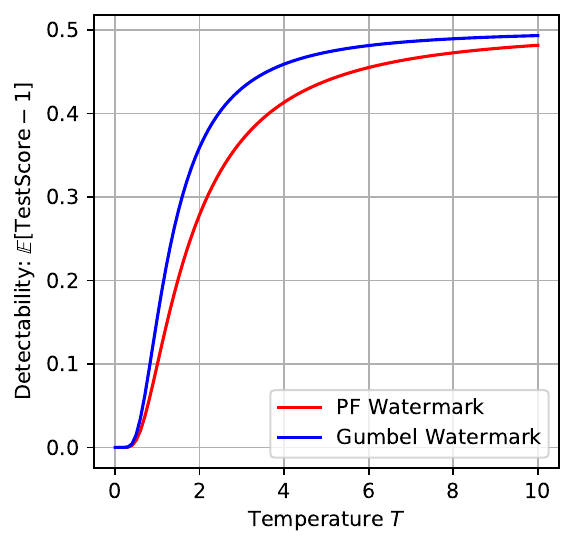}
\includegraphics[width=0.49\linewidth]{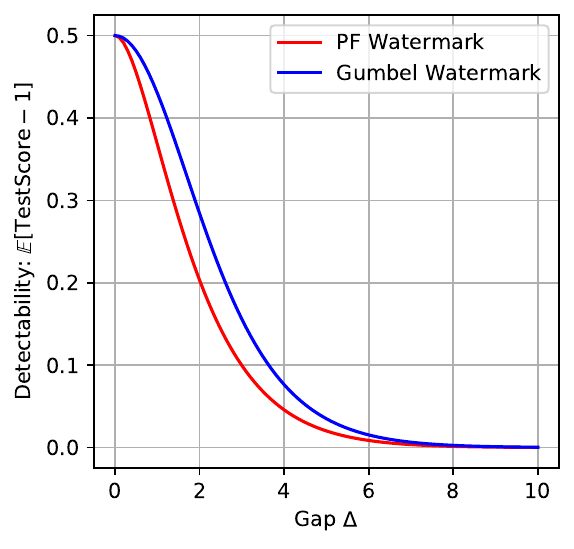}
    \caption{Comparing the \emph{detectability} of PF watermark vs Gumbel watermark  using Example~\ref{ex:twotoken_watermark}. On the left, we fix the Gap $\Delta=3.0$ and vary $T$.  On the right, we fix $T=1.0$ and vary $\Delta$. Gumbel watermark offers higher detectability when $T$ is the same. \textbf{Very importantly} this does \emph{not} mean Gumbel outperforms PF because PF is more greedy and has less entropy. }\label{fig:detectability_fixed_stability}
\end{subfigure}
\hfill
\begin{subfigure}[t]{0.49\textwidth}
    \includegraphics[width=0.47\linewidth]{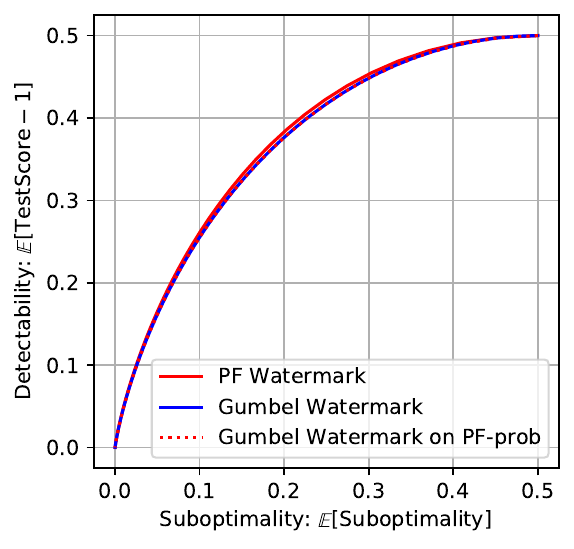}
    \includegraphics[width=0.507\linewidth]{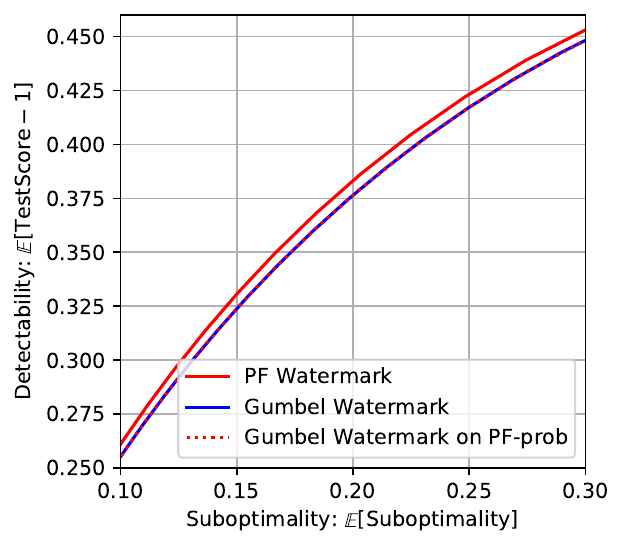}
    \caption{Comparing the \emph{detectability}-\emph{greediness} tradeoff of PF watermark vs Gumbel watermark in the two token case. The Gap $\Delta=3.0$, both curves are traced out by varying the temperature $T$ -- with a ``zoomed-in'' view on the RHS. It is clear from the figure that PF outperforms Gumbel on the tradeoff curve.}\label{fig:detectability_vs_suboptimality}
\end{subfigure}
\caption{Comparing the \emph{detectability} of PF watermark vs Gumbel watermark  using Example~\ref{ex:twotoken_watermark}.}
\end{figure}


Figure~\ref{fig:detectability_fixed_stability} indicates the PF watermark does not beat the Gumbel watermark in terms of \emph{detectability} when $T$ is fixed. This should not be surprising since for the same temperature, PF watermark is better at optimizing (recall Example~\ref{ex:twotoken} and Figure~\ref{fig:suboptimality_fixed_temprature}), thus naturally the resulting distribution has less entropy to be exploited by the watermarking scheme. 

A more fair comparison, would be to increase the temperature for PF watermark appropriately so we compare \emph{detectability} when the \emph{suboptimality} is aligned. This is shown in Figure~\ref{fig:detectability_vs_suboptimality}. In fact we have added a second baseline that apply Gumbel watermark to the \emph{induced sampling distribution from PF-decoding} (shown as the dotted line).
The distribution induced by PF does not have a simple form, but in our special case, it was worked out in Example~\ref{ex:twotoken}.    

As we can see in Figure~\ref{fig:detectability_vs_suboptimality}, the PF watermark is never worse and even has a slight advantage in the middle. To say it differently, to achieve the same suboptimality, the PF watermark can afford to use a larger temperature, which not only improves the stability parameter but also compensates it sufficiently on the detectability front to outperform the Gumbel watermark. In practice, we expect PF watermark to be as effective as the Gumbel watermark, and could even be a bit better (if the temperature parameter is chosen appropriately).


In conclusion, we showed that the watermarked version of PF-decoder is computationally indistinguishable from the original version of PF-decoder.   Meanwhile,  the test score of the PF watermark is qualitatively similar to that of the Gumbel-watermark (and identical in some cases). It is likely to produce similar detectability to the Gumbel watermark, while enjoying the performance boost that comes from replacing softmax sampling with PF.

\section{Experiments}
\label{sec:exp}

\paragraph{Datasets and models.} We utilize two long-form text datasets in our experiments: the Colossal Clean Crawled Corpus (C4) dataset \citep{raffel2020exploring} for open-ended text completion generation, and the Alpaca dataset \citep{alpaca} for question-answering tasks. Our primary language model is the state-of-the-art open-source model Llama-2 with 7 billion parameters. Specifically, we use the Llama-2-7B-chat model for question-answering tasks on the Alpaca dataset. For text completion tasks on the C4 dataset, we employ the base model Llama-2-7B. Furthermore, to evaluate the universal applicability of smaller models, we also assess the performance of the TinyLlama-1.1B model\footnote{\url{https://huggingface.co/TinyLlama/TinyLlama-1.1B-Chat-v1.0}}  \citep{zhang2024tinyllama}.

\paragraph{Evaluation metrics.} 
We calculate perplexity scores from different models, using Llama2-7B to compute PPL1 and Llama2-13B to compute PPL2. We also compute MAUVE scores to measure the distributional similarity between model generations and human text as another metric for text quality \citep{pillutla2021mauve}. To evaluate repetitiveness, we compute seq-rep-5 across generations, which is the average repetition rate of duplicate 5-grams in a sequence \citep{welleck2019neural}. 
For the watermark evaluation, maintaining a low false positive rate is crucial to avoid misclassifying unwatermarked text as watermarked. Therefore, we set the false positive rates at 1\% and 10\% for all watermark detection algorithms, adjusting the detection threshold accordingly. We report true positive rate (TPR) and F1 scores to measure the watermark detectability. We compared the well-known Gumbel Watermark (Gumbel WM) and Green-Red Watermark (KGW WM) as our main baselines. Experiments were conducted using Nvidia A6000 GPUs. For the details of the experiment setting, please refer to the Appendix \ref{sec:app_exp}.

\begin{table}[t]
\centering
\setlength{\tabcolsep}{3.5pt}
\caption{Text generation results for different methods. The true positive rate (TPR) is calculated under 0.01 false positive rate (FPR). PPL1 refers to perplexity measured by Llama2-7B models. PPL2 is perplexity calculated by the Llama2-13B model. For general text generation, PF decoding produces significantly lower perplexity compared to sampling. For watermarking methods, PF watermark also produces lower perplexity compared to KGW watermark and Gumbel watermark.}
\label{tab:main_table}
\resizebox{\linewidth}{!}{
\begin{tabular}{lccrrcc|lccrrcc}
\toprule
\textbf{Method} & \textbf{AUC}$\uparrow$ & \textbf{TPR}$\uparrow$ & \multicolumn{1}{c}{\textbf{PPL1}$\downarrow$} & \multicolumn{1}{c}{\textbf{PPL2}$\downarrow$}& \textbf{Seq-rep-5}$\downarrow$ & \textbf{MAUVE}$\uparrow$ & \textbf{Method} & \textbf{AUC}$\uparrow$ & \textbf{TPR}$\uparrow$ & \multicolumn{1}{c}{\textbf{PPL1}$\downarrow$} & \multicolumn{1}{c}{\textbf{PPL2}$\downarrow$}& \textbf{Seq-rep-5}$\downarrow$ & \textbf{MAUVE}$\uparrow$ \\
\midrule
\multicolumn{7}{l}{\textbf{C4, T=1.0, Llama2-7B}} & \multicolumn{7}{l}{\textbf{C4, T=0.8, Llama2-7B}}\\
~~~Greedy & - & - & \fullres{1.14}{0.01} & \fullres{1.24}{0.03} & 0.56 & 0.05 & ~~~Greedy & - & - & \fullres{1.28}{0.02} & \fullres{1.75}{0.03} & 0.12 & 0.93 \\
~~~Sampling & - & - & \fullres{12.47}{0.32} & \fullres{15.31}{0.41} & 0.02 & 0.98 & ~~~Sampling &  - & - & \fullres{4.23}{0.06} & \fullres{4.91}{0.08} & 0.06 & 1.00 \\
~~~PF & - & - & \fullres{8.94}{0.20} & \fullres{10.75}{0.25} & 0.03 & 0.90 & ~~~PF     & - & - & \fullres{3.54}{0.06} & \fullres{4.11}{0.08} & 0.10 & 0.92 \\
~~~KGW WM & 0.989 & 0.991 & \fullres{16.62}{0.38} & \fullres{20.62}{0.49} & 0.01 & 1.00 & ~~~KGW WM & 0.995 & 0.991 & \fullres{5.78}{0.08} & \fullres{6.77}{0.11} & 0.03 & 0.99\\
~~~Gumbel WM & 0.997 & 0.988 & \fullres{11.41}{0.27}& \fullres{14.12}{0.36} & 0.04 & 0.93 & ~~~Gumbel WM & 0.995 & 0.982 & \fullres{4.03}{0.07} & \fullres{4.71}{0.09} & 0.10 & 1.00 \\
~~~PF WM & 0.995 & 0.984 & \fullres{8.33}{0.20} & \fullres{10.28}{0.29} & 0.05 & 0.99 & ~~~PF WM & 0.993 & 0.980 & \fullres{3.38}{0.07} & \fullres{3.99}{0.10}& 0.13 & 1.00 \\
\midrule
\multicolumn{7}{l}{\textbf{Alpaca, T=1.0, Llama2-7B-Chat}} & \multicolumn{7}{l}{\textbf{Alpaca, T=1.0, TinyLlama-1.1B-Chat}}\\
~~~Greedy & - & - & \fullres{1.28}{0.02} & \fullres{1.75}{0.03} & 0.12 & 0.93 & ~~~Greedy &  - & - & \fullres{1.41}{0.01} & \fullres{1.66}{0.02} & 0.30 & 0.99 \\  
~~~Sampling &  - & - & \fullres{1.74}{0.02} & \fullres{2.41}{0.04} & 0.09 & 0.86 & ~~~Sampling &  - & - & \fullres{2.73}{0.04} & \fullres{3.71}{0.06} & 0.11 & 1.00 \\
~~~PF &  - & - & \fullres{1.65}{0.02} & \fullres{2.30}{0.04}& 0.09 & 0.98 & ~~~PF & - & - & \fullres{2.53}{0.03} & \fullres{3.44}{0.06} & 0.12 & 0.98 \\
~~~KGW WM & 0.961 & 0.596 & \fullres{2.20}{0.04} & \fullres{3.00}{0.06}& 0.08 & 0.93 & ~~~KGW WM & 0.998 & 0.991 & \fullres{3.81}{0.06} & \fullres{5.28}{0.09} & 0.07 & 0.99\\
~~~Gumbel WM & 0.986 & 0.858 & \fullres{1.70}{0.02} & \fullres{2.35}{0.04} & 0.10 & 0.93 & ~~~Gumbel WM & 1.000 & 0.995 & \fullres{2.67}{0.04} & \fullres{3.58}{0.06} & 0.12 & 1.00 \\
~~~PF WM & 0.979 & 0.810 & \fullres{1.69}{0.03} & \fullres{2.37}{0.04} & 0.10 & 1.00 & ~~~PF WM & 0.999 & 0.986  & \fullres{2.36}{0.04} & \fullres{3.15}{0.07} & 0.14 & 0.94 \\
\bottomrule
\end{tabular}
}
\end{table}

\begin{figure}[t]
    \centering
\begin{subfigure}[t]{0.49\textwidth}
    \includegraphics[width=0.95\linewidth]{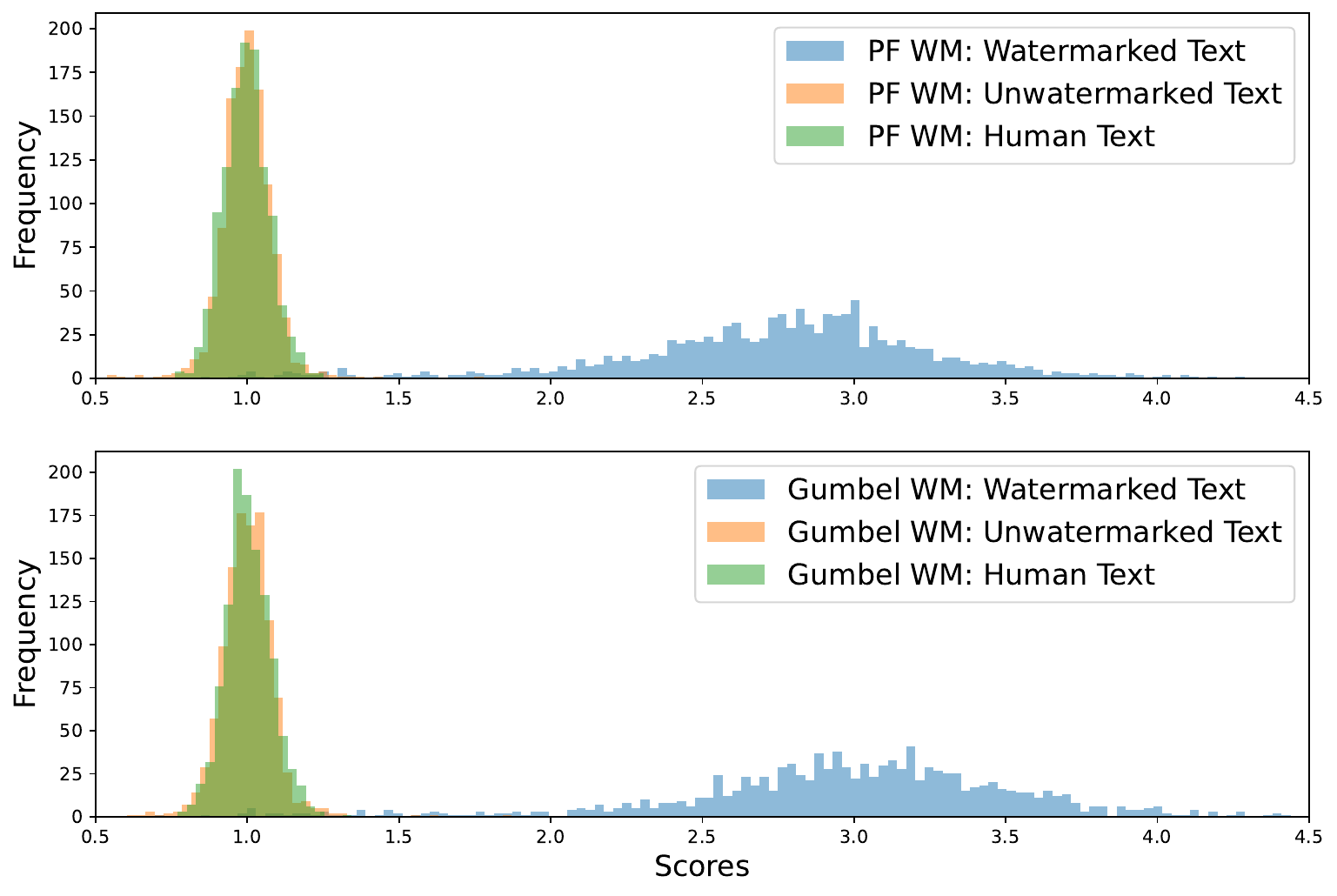}
    \caption{TestScore distribution. We calculate the average TestScore of the PF watermark and Gumbel watermark using Llama2-7B (T=1.0) on the C4 dataset. The length of the suspect texts is fixed at 200 tokens. A clear gap emerges between positive samples (watermarked) and negative samples (unwatermarked and human-written), indicating the watermark detectability.}
    \label{fig:score_dis}
\end{subfigure}
\hfill
\begin{subfigure}[t]{0.49\textwidth}
    \centering
    \includegraphics[width=0.95\linewidth]{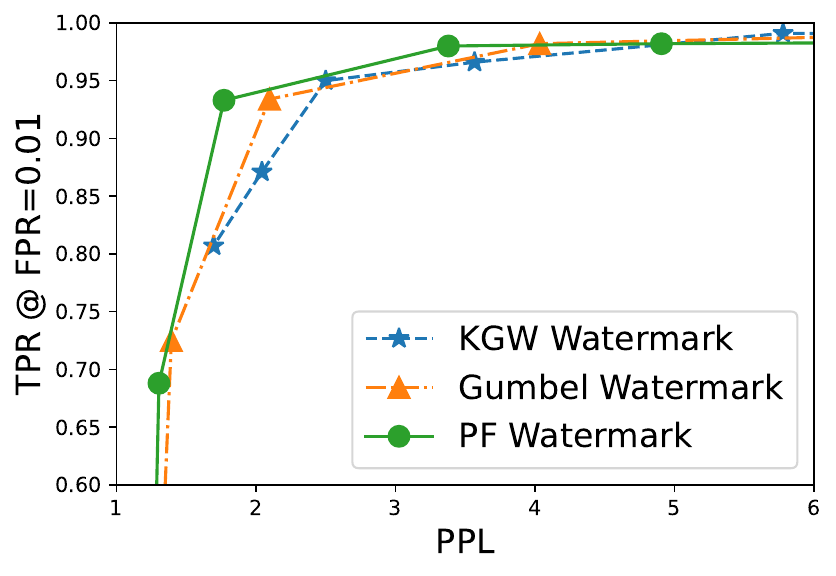}
    \caption{Trade-off between detection accuracy (TPR at FPR=0.01) and text quality (PPL) across three watermark configurations on the C4 dataset, with temperature settings ranging from 0.2 to 1.0. Each data point represents the outcome for 500 watermarked texts. The PF watermark achieves the optimal balance of the highest detection accuracy and lowest perplexity.}    \label{fig:tradeoff}
\end{subfigure}
\caption{Comparison of PF and Gumbel watermarks on real data.}
\end{figure}

\paragraph{Text generation performance.}
Table \ref{tab:main_table} shows the text perplexity of generated samples from different LLMs evaluated on two datasets. Using the same temperature, we find that PF decoding produces significantly lower perplexity compared to sampling. Although greedy decoding has the lowest perplexity, it suffers from heavy repetition, as indicated by its high seq-rep-5 score and low MAUVE score. We observe that for question-answering tasks, the perplexity is lower, likely due to the fixed form of answers and lower entropy of the text generation. Table \ref{tab:examples1} shows an example prompt and responses generated by different decoding methods.


\paragraph{Watermarking results.}
We compare the results of PF watermark with those of the Gumbel Watermark (Gumbel WM) and the Green-Red watermark (KGW WM). In Figure \ref{fig:score_dis}, we present the distribution of detection scores for the PF watermark. The PF watermark demonstrates clear detectability between positive and negative samples. The results of the watermark generation are shown in Table \ref{tab:main_table} and Figure \ref{fig:tradeoff}. The PF watermark achieves the best balance of the highest detection accuracy and lowest perplexity, compared to the KGW WM and the Gumbel WM. Notably, the perplexity of the PF watermark is close to that of the PF sampling, indicating that the watermarking process does not significantly impact the quality of the generated text. All watermarking methods achieved near-perfect detection accuracy on the C4 dataset. Besides, the detection results for the small TinyLlama model are also good, demonstrating the universal applicability of the PF watermark.


\begin{figure}[ht]
\centering
\includegraphics[width=0.55\linewidth]{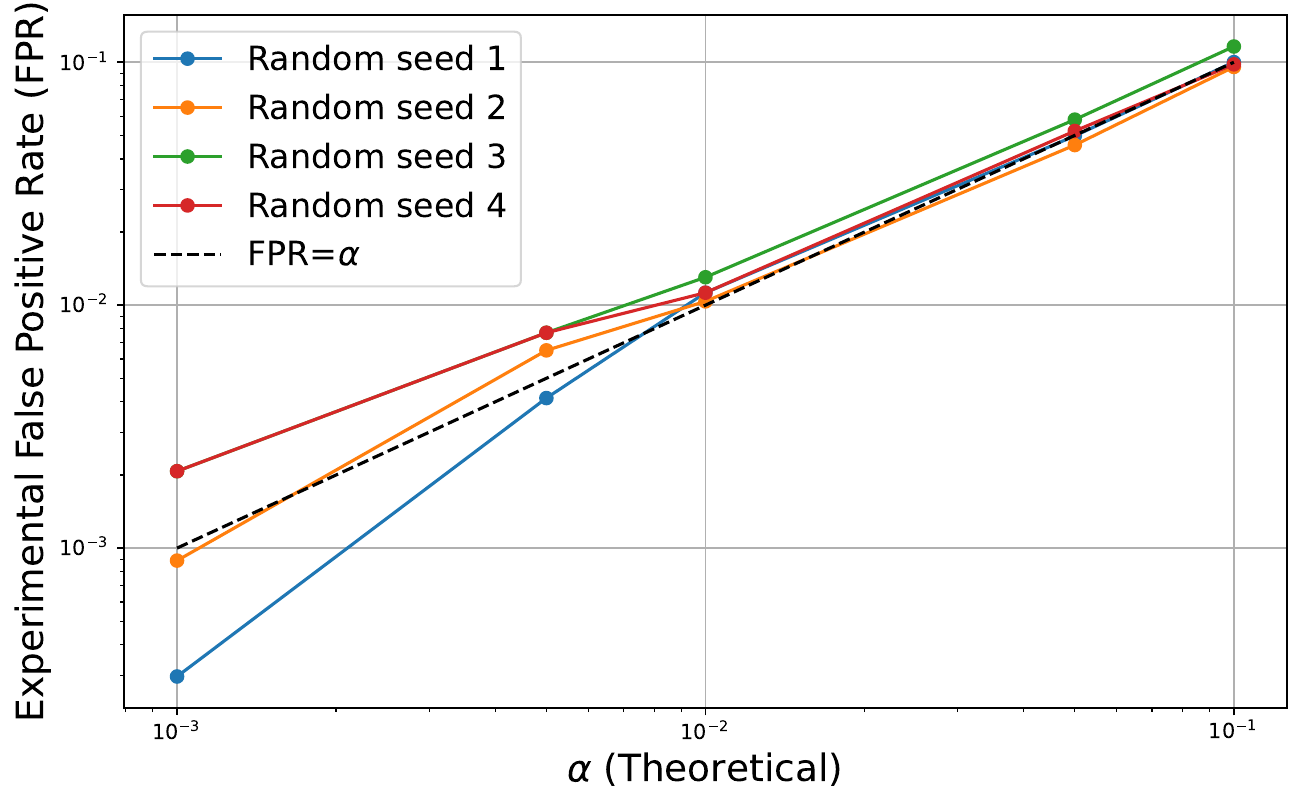}
\vspace{-0.5em}
    \caption{Comparison of empirical and theoretical false positive rates with different watermark keys. We can see that the second statement of Theorem 4.3 correctly controls the Type I error in practice.}
    \label{fig:type1}
\end{figure}
\paragraph{Controlling the false positive rate.}
The key strength of PF watermark is its ability to precisely control the false positive rate (FPR) during detection. We validate this by conducting experiments using negative examples from diverse datasets (C4, Alpaca, unwatermarked) and different random keys. As Figure \ref{fig:type1} shows, the empirical false positive rates align tightly with the theoretical $\alpha$ values across different settings. This demonstrates PF watermark's effectiveness in precisely controlling the FPR as intended.

\paragraph{Additional watermarking results.} For a text watermarking design to be effective, it should be able to withstand paraphrasing attacks that an adversary may attempt to modify the watermarked text. Furthermore, the watermark should be detectable even with shorter text lengths. In Appendices \ref{sec:robust}.1 and \ref{sec:length}.2, we present additional empirical results for the PF watermark, demonstrating its robustness to paraphrasing and editing attacks. The results also show that the PF watermark can still be detected even when the length of the text is reduced to only 30 tokens.

\section{Conclusion}
\label{sec:dis}

We introduce Permute-and-Flip (PF) decoding, a new decoding method for large language models that enjoys the same -- perturbation-stability guarantees as softmax sampling while achieving substantially lower perplexity. We design a tailored watermarking scheme (PF watermark) for PF decoding that enables precise control over false positive rates while retaining high true positive rates. Our experiments demonstrate that the PF watermark achieves the best balance of the highest detection accuracy and lowest perplexity. All these intriguing properties make PF decoding a promising new approach for practical applications of large language models.


\bibliography{custom}

\begin{thebibliography}{63}
\providecommand{\natexlab}[1]{#1}
\providecommand{\url}[1]{\texttt{#1}}
\expandafter\ifx\csname urlstyle\endcsname\relax
  \providecommand{\doi}[1]{doi: #1}\else
  \providecommand{\doi}{doi: \begingroup \urlstyle{rm}\Url}\fi

\bibitem[Aaronson(2023)]{aaronson}
Scott Aaronson.
\newblock Simons institute talk on watermarking of large language models, 2023.
\newblock URL \url{https://simons.berkeley.edu/talks/scott-aaronson-ut-austin-openai-2023-08-17}.

\bibitem[Anderson et~al.(2017)Anderson, Fernando, Johnson, and Gould]{anderson2016guided}
Peter Anderson, Basura Fernando, Mark Johnson, and Stephen Gould.
\newblock Guided open vocabulary image captioning with constrained beam search.
\newblock In \emph{Proceedings of the 2017 Conference on Empirical Methods in Natural Language Processing}, pp.\  936--945, Copenhagen, Denmark, 2017. Association for Computational Linguistics.
\newblock \doi{10.18653/v1/D17-1098}.
\newblock URL \url{https://aclanthology.org/D17-1098}.

\bibitem[Atallah et~al.(2001)Atallah, Raskin, Crogan, Hempelmann, Kerschbaum, Mohamed, and Naik]{Atallah2001NaturalLW}
Mikhail~J. Atallah, Victor Raskin, Michael Crogan, Christian~F. Hempelmann, Florian Kerschbaum, Dina Mohamed, and Sanket Naik.
\newblock Natural language watermarking: Design, analysis, and a proof-of-concept implementation.
\newblock In \emph{Information Hiding}, 2001.

\bibitem[Bai et~al.(2022)Bai, Kadavath, Kundu, Askell, Kernion, Jones, Chen, Goldie, Mirhoseini, McKinnon, et~al.]{bai2022constitutional}
Yuntao Bai, Saurav Kadavath, Sandipan Kundu, Amanda Askell, Jackson Kernion, Andy Jones, Anna Chen, Anna Goldie, Azalia Mirhoseini, Cameron McKinnon, et~al.
\newblock Constitutional ai: Harmlessness from ai feedback.
\newblock \emph{ArXiv preprint}, abs/2212.08073, 2022.
\newblock URL \url{https://arxiv.org/abs/2212.08073}.

\bibitem[Bao et~al.(2024)Bao, Zhao, Teng, Yang, and Zhang]{Bao2023FastDetectGPTEZ}
Guangsheng Bao, Yanbin Zhao, Zhiyang Teng, Linyi Yang, and Yue Zhang.
\newblock Fast-detectgpt: Efficient zero-shot detection of machine-generated text via conditional probability curvature.
\newblock In \emph{The Twelfth International Conference on Learning Representations, {ICLR} 2024, Vienna, Austria, May 7-11, 2024}. OpenReview.net, 2024.
\newblock URL \url{https://openreview.net/forum?id=Bpcgcr8E8Z}.

\bibitem[Chen et~al.(2023{\natexlab{a}})Chen, Borgeaud, Irving, Lespiau, Sifre, and Jumper]{chen2023accelerating}
Charlie Chen, Sebastian Borgeaud, Geoffrey Irving, Jean-Baptiste Lespiau, Laurent Sifre, and John Jumper.
\newblock Accelerating large language model decoding with speculative sampling.
\newblock \emph{ArXiv preprint}, abs/2302.01318, 2023{\natexlab{a}}.
\newblock URL \url{https://arxiv.org/abs/2302.01318}.

\bibitem[Chen et~al.(2018)Chen, Gilroy, Maletti, May, and Knight]{Chen2017RecurrentNN}
Yining Chen, Sorcha Gilroy, Andreas Maletti, Jonathan May, and Kevin Knight.
\newblock Recurrent neural networks as weighted language recognizers.
\newblock In \emph{Proceedings of the 2018 Conference of the North {A}merican Chapter of the Association for Computational Linguistics: Human Language Technologies, Volume 1 (Long Papers)}, pp.\  2261--2271, New Orleans, Louisiana, 2018. Association for Computational Linguistics.
\newblock \doi{10.18653/v1/N18-1205}.
\newblock URL \url{https://aclanthology.org/N18-1205}.

\bibitem[Chen et~al.(2023{\natexlab{b}})Chen, Kang, Zhai, Li, Singh, and Ramakrishnan]{Chen2023GPTSentinelDH}
Yutian Chen, Hao Kang, Vivian Zhai, Liang Li, Rita Singh, and Bhiksha Ramakrishnan.
\newblock Gpt-sentinel: Distinguishing human and chatgpt generated content.
\newblock \emph{ArXiv preprint}, abs/2305.07969, 2023{\natexlab{b}}.
\newblock URL \url{https://arxiv.org/abs/2305.07969}.

\bibitem[Christ et~al.(2024)Christ, Gunn, and Zamir]{christ2023undetectable}
Miranda Christ, Sam Gunn, and Or~Zamir.
\newblock Undetectable watermarks for language models.
\newblock In \emph{The Thirty Seventh Annual Conference on Learning Theory, June 30 - July 3, 2023, Edmonton, Canada}, volume 247 of \emph{Proceedings of Machine Learning Research}, pp.\  1125--1139. {PMLR}, 2024.
\newblock URL \url{https://proceedings.mlr.press/v247/christ24a.html}.

\bibitem[Ding et~al.(2021)Ding, Kifer, Steinke, Wang, Xiao, Zhang, et~al.]{ding2021permute}
Zeyu Ding, Daniel Kifer, Thomas Steinke, Yuxin Wang, Yingtai Xiao, Danfeng Zhang, et~al.
\newblock The permute-and-flip mechanism is identical to report-noisy-max with exponential noise.
\newblock \emph{ArXiv preprint}, abs/2105.07260, 2021.
\newblock URL \url{https://arxiv.org/abs/2105.07260}.

\bibitem[Dwork et~al.(2006)Dwork, McSherry, Nissim, and Smith]{dwork2006calibrating}
Cynthia Dwork, Frank McSherry, Kobbi Nissim, and Adam Smith.
\newblock Calibrating noise to sensitivity in private data analysis.
\newblock In \emph{Theory of cryptography}, pp.\  265--284. Springer, 2006.

\bibitem[Fernandez et~al.(2023)Fernandez, Chaffin, Tit, Chappelier, and Furon]{Fernandez2023ThreeBT}
Pierre Fernandez, Antoine Chaffin, Karim Tit, Vivien Chappelier, and Teddy Furon.
\newblock Three bricks to consolidate watermarks for large language models.
\newblock \emph{2023 IEEE International Workshop on Information Forensics and Security (WIFS)}, pp.\  1--6, 2023.

\bibitem[GPTZero(2023)]{GPTZero}
GPTZero.
\newblock Gptzero: More than an ai detector preserve what's human.
\newblock \emph{GPTZero website}, 2023.
\newblock URL \url{https://gptzero.me/}.

\bibitem[Gumbel(1948)]{gumbel1948statistical}
Emil~Julius Gumbel.
\newblock \emph{Statistical theory of extreme values and some practical applications: a series of lectures}, volume~33.
\newblock US Government Printing Office, 1948.

\bibitem[Hokamp \& Liu(2017)Hokamp and Liu]{hokamp2017lexically}
Chris Hokamp and Qun Liu.
\newblock Lexically constrained decoding for sequence generation using grid beam search.
\newblock In \emph{Proceedings of the 55th Annual Meeting of the Association for Computational Linguistics (Volume 1: Long Papers)}, pp.\  1535--1546, Vancouver, Canada, 2017. Association for Computational Linguistics.
\newblock \doi{10.18653/v1/P17-1141}.
\newblock URL \url{https://aclanthology.org/P17-1141}.

\bibitem[Holtzman et~al.(2020)Holtzman, Buys, Du, Forbes, and Choi]{holtzman2019curious}
Ari Holtzman, Jan Buys, Li~Du, Maxwell Forbes, and Yejin Choi.
\newblock The curious case of neural text degeneration.
\newblock In \emph{8th International Conference on Learning Representations, {ICLR} 2020, Addis Ababa, Ethiopia, April 26-30, 2020}. OpenReview.net, 2020.
\newblock URL \url{https://openreview.net/forum?id=rygGQyrFvH}.

\bibitem[Hu et~al.(2023)Hu, Chen, and Ho]{Hu2023RADARRA}
Xiaomeng Hu, Pin{-}Yu Chen, and Tsung{-}Yi Ho.
\newblock {RADAR:} robust ai-text detection via adversarial learning.
\newblock In \emph{Advances in Neural Information Processing Systems 36: Annual Conference on Neural Information Processing Systems 2023, NeurIPS 2023, New Orleans, LA, USA, December 10 - 16, 2023}, 2023.
\newblock URL \url{http://papers.nips.cc/paper\_files/paper/2023/hash/30e15e5941ae0cdab7ef58cc8d59a4ca-Abstract-Conference.html}.

\bibitem[Hu et~al.(2024)Hu, Chen, Wu, Wu, Zhang, and Huang]{Hu2023UnbiasedWF}
Zhengmian Hu, Lichang Chen, Xidong Wu, Yihan Wu, Hongyang Zhang, and Heng Huang.
\newblock Unbiased watermark for large language models.
\newblock In \emph{The Twelfth International Conference on Learning Representations, {ICLR} 2024, Vienna, Austria, May 7-11, 2024}. OpenReview.net, 2024.
\newblock URL \url{https://openreview.net/forum?id=uWVC5FVidc}.

\bibitem[Ippolito et~al.(2019)Ippolito, Kriz, Sedoc, Kustikova, and Callison-Burch]{ippolito2019comparison}
Daphne Ippolito, Reno Kriz, Jo{\~a}o Sedoc, Maria Kustikova, and Chris Callison-Burch.
\newblock Comparison of diverse decoding methods from conditional language models.
\newblock In \emph{Proceedings of the 57th Annual Meeting of the Association for Computational Linguistics}, pp.\  3752--3762, Florence, Italy, 2019. Association for Computational Linguistics.
\newblock \doi{10.18653/v1/P19-1365}.
\newblock URL \url{https://aclanthology.org/P19-1365}.

\bibitem[Kirchenbauer et~al.(2023)Kirchenbauer, Geiping, Wen, Katz, Miers, and Goldstein]{kirchenbauer2023watermark}
John Kirchenbauer, Jonas Geiping, Yuxin Wen, Jonathan Katz, Ian Miers, and Tom Goldstein.
\newblock A watermark for large language models.
\newblock In \emph{International Conference on Machine Learning, {ICML} 2023, 23-29 July 2023, Honolulu, Hawaii, {USA}}, volume 202 of \emph{Proceedings of Machine Learning Research}, pp.\  17061--17084. {PMLR}, 2023.
\newblock URL \url{https://proceedings.mlr.press/v202/kirchenbauer23a.html}.

\bibitem[Krishna et~al.(2023)Krishna, Song, Karpinska, Wieting, and Iyyer]{Krishna2023ParaphrasingED}
Kalpesh Krishna, Yixiao Song, Marzena Karpinska, John Wieting, and Mohit Iyyer.
\newblock Paraphrasing evades detectors of ai-generated text, but retrieval is an effective defense.
\newblock In \emph{Advances in Neural Information Processing Systems 36: Annual Conference on Neural Information Processing Systems 2023, NeurIPS 2023, New Orleans, LA, USA, December 10 - 16, 2023}, 2023.
\newblock URL \url{http://papers.nips.cc/paper\_files/paper/2023/hash/575c450013d0e99e4b0ecf82bd1afaa4-Abstract-Conference.html}.

\bibitem[Kuditipudi et~al.(2024)Kuditipudi, Thickstun, Hashimoto, and Liang]{Kuditipudi2023RobustDW}
Rohith Kuditipudi, John Thickstun, Tatsunori Hashimoto, and Percy Liang.
\newblock Robust distortion-free watermarks for language models.
\newblock \emph{Trans. Mach. Learn. Res.}, 2024, 2024.
\newblock URL \url{https://openreview.net/forum?id=FpaCL1MO2C}.

\bibitem[Lewis et~al.(2020)Lewis, Perez, Piktus, Petroni, Karpukhin, Goyal, K{\"{u}}ttler, Lewis, Yih, Rockt{\"{a}}schel, Riedel, and Kiela]{lewis2020retrieval}
Patrick S.~H. Lewis, Ethan Perez, Aleksandra Piktus, Fabio Petroni, Vladimir Karpukhin, Naman Goyal, Heinrich K{\"{u}}ttler, Mike Lewis, Wen{-}tau Yih, Tim Rockt{\"{a}}schel, Sebastian Riedel, and Douwe Kiela.
\newblock Retrieval-augmented generation for knowledge-intensive {NLP} tasks.
\newblock In \emph{Advances in Neural Information Processing Systems 33: Annual Conference on Neural Information Processing Systems 2020, NeurIPS 2020, December 6-12, 2020, virtual}, 2020.
\newblock URL \url{https://proceedings.neurips.cc/paper/2020/hash/6b493230205f780e1bc26945df7481e5-Abstract.html}.

\bibitem[Li et~al.(2023)Li, Holtzman, Fried, Liang, Eisner, Hashimoto, Zettlemoyer, and Lewis]{li2022contrastive}
Xiang~Lisa Li, Ari Holtzman, Daniel Fried, Percy Liang, Jason Eisner, Tatsunori Hashimoto, Luke Zettlemoyer, and Mike Lewis.
\newblock Contrastive decoding: Open-ended text generation as optimization.
\newblock In \emph{Proceedings of the 61st Annual Meeting of the Association for Computational Linguistics (Volume 1: Long Papers)}, pp.\  12286--12312, Toronto, Canada, 2023. Association for Computational Linguistics.
\newblock \doi{10.18653/v1/2023.acl-long.687}.
\newblock URL \url{https://aclanthology.org/2023.acl-long.687}.

\bibitem[Liang et~al.(2023)Liang, Y{\"{u}}ksekg{\"{o}}n{\"{u}}l, Mao, Wu, and Zou]{Liang2023GPTDA}
Weixin Liang, Mert Y{\"{u}}ksekg{\"{o}}n{\"{u}}l, Yining Mao, Eric Wu, and James Zou.
\newblock {GPT} detectors are biased against non-native english writers.
\newblock \emph{Patterns}, 4\penalty0 (7):\penalty0 100779, 2023.
\newblock \doi{10.1016/J.PATTER.2023.100779}.
\newblock URL \url{https://doi.org/10.1016/j.patter.2023.100779}.

\bibitem[Lin et~al.(2021)Lin, Dang, Rahouti, and Xiong]{lin2021ml}
Jing Lin, Long Dang, Mohamed Rahouti, and Kaiqi Xiong.
\newblock Ml attack models: adversarial attacks and data poisoning attacks.
\newblock \emph{ArXiv preprint}, abs/2112.02797, 2021.
\newblock URL \url{https://arxiv.org/abs/2112.02797}.

\bibitem[Liu et~al.(2023)Liu, Zhang, Wang, Pu, Lan, and Shen]{Liu2022CoCoCM}
Xiaoming Liu, Zhaohan Zhang, Yichen Wang, Hang Pu, Yu~Lan, and Chao Shen.
\newblock Coco: Coherence-enhanced machine-generated text detection under low resource with contrastive learning.
\newblock In \emph{Proceedings of the 2023 Conference on Empirical Methods in Natural Language Processing}, pp.\  16167--16188, Singapore, December 2023. Association for Computational Linguistics.
\newblock \doi{10.18653/v1/2023.emnlp-main.1005}.
\newblock URL \url{https://aclanthology.org/2023.emnlp-main.1005/}.

\bibitem[Lu et~al.(2022)Lu, Welleck, West, Jiang, Kasai, Khashabi, Le~Bras, Qin, Yu, Zellers, Smith, and Choi]{lu2021neurologic}
Ximing Lu, Sean Welleck, Peter West, Liwei Jiang, Jungo Kasai, Daniel Khashabi, Ronan Le~Bras, Lianhui Qin, Youngjae Yu, Rowan Zellers, Noah~A. Smith, and Yejin Choi.
\newblock {N}euro{L}ogic a*esque decoding: Constrained text generation with lookahead heuristics.
\newblock In \emph{Proceedings of the 2022 Conference of the North American Chapter of the Association for Computational Linguistics: Human Language Technologies}, pp.\  780--799, Seattle, United States, 2022. Association for Computational Linguistics.
\newblock \doi{10.18653/v1/2022.naacl-main.57}.
\newblock URL \url{https://aclanthology.org/2022.naacl-main.57}.

\bibitem[McKenna \& Sheldon(2020)McKenna and Sheldon]{mckenna2020permute}
Ryan McKenna and Daniel~R. Sheldon.
\newblock Permute-and-flip: {A} new mechanism for differentially private selection.
\newblock In \emph{Advances in Neural Information Processing Systems 33: Annual Conference on Neural Information Processing Systems 2020, NeurIPS 2020, December 6-12, 2020, virtual}, 2020.
\newblock URL \url{https://proceedings.neurips.cc/paper/2020/hash/01e00f2f4bfcbb7505cb641066f2859b-Abstract.html}.

\bibitem[McSherry \& Talwar(2007)McSherry and Talwar]{mcsherry2007mechanism}
Frank McSherry and Kunal Talwar.
\newblock Mechanism design via differential privacy.
\newblock In \emph{Symposium on Foundations of Computer Science (FOCS'07)}, pp.\  94--103. IEEE, 2007.

\bibitem[Meister et~al.(2020)Meister, Salesky, and Cotterell]{meister2020generalized}
Clara Meister, Elizabeth Salesky, and Ryan Cotterell.
\newblock Generalized entropy regularization or: There{'}s nothing special about label smoothing.
\newblock In \emph{Proceedings of the 58th Annual Meeting of the Association for Computational Linguistics}, pp.\  6870--6886, Online, 2020. Association for Computational Linguistics.
\newblock \doi{10.18653/v1/2020.acl-main.615}.
\newblock URL \url{https://aclanthology.org/2020.acl-main.615}.

\bibitem[Mitchell et~al.(2023)Mitchell, Lee, Khazatsky, Manning, and Finn]{Mitchell2023DetectGPTZM}
Eric Mitchell, Yoonho Lee, Alexander Khazatsky, Christopher~D. Manning, and Chelsea Finn.
\newblock Detectgpt: Zero-shot machine-generated text detection using probability curvature.
\newblock In \emph{International Conference on Machine Learning, {ICML} 2023, 23-29 July 2023, Honolulu, Hawaii, {USA}}, volume 202 of \emph{Proceedings of Machine Learning Research}, pp.\  24950--24962. {PMLR}, 2023.
\newblock URL \url{https://proceedings.mlr.press/v202/mitchell23a.html}.

\bibitem[OpenAI(2022)]{OpenAI2022ChatGPT}
OpenAI.
\newblock Chatgpt: Optimizing language models for dialogue.
\newblock \emph{OpenAI blog}, 2022.
\newblock URL \url{https://openai.com/blog/chatgpt/}.

\bibitem[OpenAI(2023{\natexlab{a}})]{OpenAI2023Detect}
OpenAI.
\newblock New ai classifier for indicating ai-written text.
\newblock \emph{OpenAI blog}, 2023{\natexlab{a}}.
\newblock URL \url{https://openai.com/blog/new-ai-classifier-for-indicating-ai-written-text}.

\bibitem[OpenAI(2023{\natexlab{b}})]{OpenAI2023GPT4TR}
OpenAI.
\newblock Gpt-4 technical report.
\newblock \emph{ArXiv preprint}, abs/2303.08774, 2023{\natexlab{b}}.
\newblock URL \url{https://arxiv.org/abs/2303.08774}.

\bibitem[Pillutla et~al.(2021)Pillutla, Swayamdipta, Zellers, Thickstun, Welleck, Choi, and Harchaoui]{pillutla2021mauve}
Krishna Pillutla, Swabha Swayamdipta, Rowan Zellers, John Thickstun, Sean Welleck, Yejin Choi, and Za{\"{\i}}d Harchaoui.
\newblock {MAUVE:} measuring the gap between neural text and human text using divergence frontiers.
\newblock In \emph{Advances in Neural Information Processing Systems 34: Annual Conference on Neural Information Processing Systems 2021, NeurIPS 2021, December 6-14, 2021, virtual}, pp.\  4816--4828, 2021.
\newblock URL \url{https://proceedings.neurips.cc/paper/2021/hash/260c2432a0eecc28ce03c10dadc078a4-Abstract.html}.

\bibitem[Qin et~al.(2020)Qin, Shwartz, West, Bhagavatula, Hwang, Le~Bras, Bosselut, and Choi]{qin2020backpropagation}
Lianhui Qin, Vered Shwartz, Peter West, Chandra Bhagavatula, Jena~D Hwang, Ronan Le~Bras, Antoine Bosselut, and Yejin Choi.
\newblock Backpropagation-based decoding for unsupervised counterfactual and abductive reasoning.
\newblock In \emph{Proceedings of the 2020 Conference on Empirical Methods in Natural Language Processing (EMNLP)}, pp.\  794--805, 2020.

\bibitem[Raffel et~al.(2020)Raffel, Shazeer, Roberts, Lee, Narang, Matena, Zhou, Li, and Liu]{raffel2020exploring}
Colin Raffel, Noam Shazeer, Adam Roberts, Katherine Lee, Sharan Narang, Michael Matena, Yanqi Zhou, Wei Li, and Peter~J. Liu.
\newblock Exploring the limits of transfer learning with a unified text-to-text transformer.
\newblock \emph{J. Mach. Learn. Res.}, 21:\penalty0 140:1--140:67, 2020.
\newblock URL \url{http://jmlr.org/papers/v21/20-074.html}.

\bibitem[Sadasivan et~al.(2023)Sadasivan, Kumar, Balasubramanian, Wang, and Feizi]{Sadasivan2023CanAT}
Vinu~Sankar Sadasivan, Aounon Kumar, S.~Balasubramanian, Wenxiao Wang, and Soheil Feizi.
\newblock Can ai-generated text be reliably detected?
\newblock \emph{ArXiv preprint}, abs/2303.11156, 2023.
\newblock URL \url{https://arxiv.org/abs/2303.11156}.

\bibitem[Shi et~al.(2024)Shi, Wang, Yin, Chen, Chang, and Hsieh]{Shi2023RedTL}
Zhouxing Shi, Yihan Wang, Fan Yin, Xiangning Chen, Kai-Wei Chang, and Cho-Jui Hsieh.
\newblock Red teaming language model detectors with language models.
\newblock \emph{Transactions of the Association for Computational Linguistics}, 12:\penalty0 174--189, 2024.
\newblock \doi{10.1162/tacl_a_00639}.
\newblock URL \url{https://aclanthology.org/2024.tacl-1.10}.

\bibitem[Stokel-Walker(2022)]{StokelWalker2022AIBC}
Chris Stokel-Walker.
\newblock Ai bot chatgpt writes smart essays - should professors worry?
\newblock \emph{Nature}, 2022.

\bibitem[Taori et~al.(2023)Taori, Gulrajani, Zhang, Dubois, Li, Guestrin, Liang, and Hashimoto]{alpaca}
Rohan Taori, Ishaan Gulrajani, Tianyi Zhang, Yann Dubois, Xuechen Li, Carlos Guestrin, Percy Liang, and Tatsunori~B. Hashimoto.
\newblock Stanford alpaca: An instruction-following llama model.
\newblock \url{https://github.com/tatsu-lab/stanford_alpaca}, 2023.

\bibitem[Topkara et~al.(2006)Topkara, Topkara, and Atallah]{Topkara2006TheHV}
Umut Topkara, Mercan Topkara, and Mikhail~J. Atallah.
\newblock The hiding virtues of ambiguity: quantifiably resilient watermarking of natural language text through synonym substitutions.
\newblock In \emph{Workshop on Multimedia \& Security}, 2006.

\bibitem[Touvron et~al.(2023)Touvron, Martin, Stone, Albert, Almahairi, Babaei, Bashlykov, Batra, Bhargava, Bhosale, et~al.]{touvron2023llama}
Hugo Touvron, Louis Martin, Kevin Stone, Peter Albert, Amjad Almahairi, Yasmine Babaei, Nikolay Bashlykov, Soumya Batra, Prajjwal Bhargava, Shruti Bhosale, et~al.
\newblock Llama 2: Open foundation and fine-tuned chat models.
\newblock \emph{ArXiv preprint}, abs/2307.09288, 2023.
\newblock URL \url{https://arxiv.org/abs/2307.09288}.

\bibitem[Tulchinskii et~al.(2023)Tulchinskii, Kuznetsov, Kushnareva, Cherniavskii, Nikolenko, Burnaev, Barannikov, and Piontkovskaya]{Tulchinskii2023IntrinsicDE}
Eduard Tulchinskii, Kristian Kuznetsov, Laida Kushnareva, Daniil Cherniavskii, Sergey~I. Nikolenko, Evgeny Burnaev, Serguei Barannikov, and Irina Piontkovskaya.
\newblock Intrinsic dimension estimation for robust detection of ai-generated texts.
\newblock In \emph{Advances in Neural Information Processing Systems 36: Annual Conference on Neural Information Processing Systems 2023, NeurIPS 2023, New Orleans, LA, USA, December 10 - 16, 2023}, 2023.
\newblock URL \url{http://papers.nips.cc/paper\_files/paper/2023/hash/7baa48bc166aa2013d78cbdc15010530-Abstract-Conference.html}.

\bibitem[Vijayakumar et~al.(2016)Vijayakumar, Cogswell, Selvaraju, Sun, Lee, Crandall, and Batra]{vijayakumar2016diverse}
Ashwin~K Vijayakumar, Michael Cogswell, Ramprasath~R Selvaraju, Qing Sun, Stefan Lee, David Crandall, and Dhruv Batra.
\newblock Diverse beam search: Decoding diverse solutions from neural sequence models.
\newblock \emph{ArXiv preprint}, abs/1610.02424, 2016.
\newblock URL \url{https://arxiv.org/abs/1610.02424}.

\bibitem[Weidinger et~al.(2021)Weidinger, Mellor, Rauh, Griffin, Uesato, Huang, Cheng, Glaese, Balle, Kasirzadeh, Kenton, Brown, Hawkins, Stepleton, Biles, Birhane, Haas, Rimell, Hendricks, Isaac, Legassick, Irving, and Gabriel]{Weidinger2021EthicalAS}
Laura Weidinger, John F.~J. Mellor, Maribeth Rauh, Conor Griffin, Jonathan Uesato, Po-Sen Huang, Myra Cheng, Mia Glaese, Borja Balle, Atoosa Kasirzadeh, Zachary Kenton, Sande~Minnich Brown, William~T. Hawkins, Tom Stepleton, Courtney Biles, Abeba Birhane, Julia Haas, Laura Rimell, Lisa~Anne Hendricks, William~S. Isaac, Sean Legassick, Geoffrey Irving, and Iason Gabriel.
\newblock Ethical and social risks of harm from language models.
\newblock \emph{ArXiv preprint}, abs/2112.04359, 2021.
\newblock URL \url{https://arxiv.org/abs/2112.04359}.

\bibitem[Welleck et~al.(2020)Welleck, Kulikov, Roller, Dinan, Cho, and Weston]{welleck2019neural}
Sean Welleck, Ilia Kulikov, Stephen Roller, Emily Dinan, Kyunghyun Cho, and Jason Weston.
\newblock Neural text generation with unlikelihood training.
\newblock In \emph{8th International Conference on Learning Representations, {ICLR} 2020, Addis Ababa, Ethiopia, April 26-30, 2020}. OpenReview.net, 2020.
\newblock URL \url{https://openreview.net/forum?id=SJeYe0NtvH}.

\bibitem[Wiher et~al.(2022)Wiher, Meister, and Cotterell]{wiher2022decoding}
Gian Wiher, Clara Meister, and Ryan Cotterell.
\newblock On decoding strategies for neural text generators.
\newblock \emph{Transactions of the Association for Computational Linguistics}, 10:\penalty0 997--1012, 2022.
\newblock \doi{10.1162/tacl_a_00502}.
\newblock URL \url{https://aclanthology.org/2022.tacl-1.58}.

\bibitem[Wolff(2020)]{Wolff2020AttackingNT}
Max Wolff.
\newblock Attacking neural text detectors.
\newblock \emph{ArXiv preprint}, abs/2002.11768, 2020.
\newblock URL \url{https://arxiv.org/abs/2002.11768}.

\bibitem[Wu et~al.(2024)Wu, Hu, Guo, Zhang, and Huang]{Wu2023DiPmarkAS}
Yihan Wu, Zhengmian Hu, Junfeng Guo, Hongyang Zhang, and Heng Huang.
\newblock A resilient and accessible distribution-preserving watermark for large language models.
\newblock In \emph{Forty-first International Conference on Machine Learning, {ICML} 2024, Vienna, Austria, July 21-27, 2024}. OpenReview.net, 2024.
\newblock URL \url{https://openreview.net/forum?id=c8qWiNiqRY}.

\bibitem[Yang et~al.(2024)Yang, Cheng, Wu, Petzold, Wang, and Chen]{Yang2023DNAGPTDN}
Xianjun Yang, Wei Cheng, Yue Wu, Linda~Ruth Petzold, William~Yang Wang, and Haifeng Chen.
\newblock {DNA-GPT:} divergent n-gram analysis for training-free detection of gpt-generated text.
\newblock In \emph{The Twelfth International Conference on Learning Representations, {ICLR} 2024, Vienna, Austria, May 7-11, 2024}. OpenReview.net, 2024.
\newblock URL \url{https://openreview.net/forum?id=Xlayxj2fWp}.

\bibitem[Yoo et~al.(2023)Yoo, Ahn, Jang, and Kwak]{Yoo2023RobustMN}
KiYoon Yoo, Wonhyuk Ahn, Jiho Jang, and Nojun Kwak.
\newblock Robust multi-bit natural language watermarking through invariant features.
\newblock In \emph{Proceedings of the 61st Annual Meeting of the Association for Computational Linguistics (Volume 1: Long Papers)}, pp.\  2092--2115, Toronto, Canada, 2023. Association for Computational Linguistics.
\newblock \doi{10.18653/v1/2023.acl-long.117}.
\newblock URL \url{https://aclanthology.org/2023.acl-long.117}.

\bibitem[Yu et~al.(2023)Yu, Qi, Chen, Chen, Yang, Zhu, Zhang, and Yu]{Yu2023GPTPT}
Xiao Yu, Yuang Qi, Kejiang Chen, Guoqiang Chen, Xi~Yang, Pengyuan Zhu, Weiming Zhang, and Neng~H. Yu.
\newblock Gpt paternity test: Gpt generated text detection with gpt genetic inheritance.
\newblock \emph{ArXiv preprint}, abs/2305.12519, 2023.
\newblock URL \url{https://arxiv.org/abs/2305.12519}.

\bibitem[Zellers et~al.(2019)Zellers, Holtzman, Rashkin, Bisk, Farhadi, Roesner, and Choi]{zellers2019defending}
Rowan Zellers, Ari Holtzman, Hannah Rashkin, Yonatan Bisk, Ali Farhadi, Franziska Roesner, and Yejin Choi.
\newblock Defending against neural fake news.
\newblock In \emph{Advances in Neural Information Processing Systems 32: Annual Conference on Neural Information Processing Systems 2019, NeurIPS 2019, December 8-14, 2019, Vancouver, BC, Canada}, pp.\  9051--9062, 2019.
\newblock URL \url{https://proceedings.neurips.cc/paper/2019/hash/3e9f0fc9b2f89e043bc6233994dfcf76-Abstract.html}.

\bibitem[ZeroGPT(2023)]{ZeroGPT}
ZeroGPT.
\newblock Zerogpt: Trusted gpt-4, chatgpt and ai detector tool by zerogpt.
\newblock \emph{ZeroGPT website}, 2023.
\newblock URL \url{https://www.zerogpt.com/}.

\bibitem[Zhang et~al.(2023)Zhang, Guo, Zhu, Cao, Lin, Jia, Chen, and Wu]{Zhang2023OnTS}
Hangfan Zhang, Zhimeng Guo, Huaisheng Zhu, Bochuan Cao, Lu~Lin, Jinyuan Jia, Jinghui Chen, and Di~Wu.
\newblock On the safety of open-sourced large language models: Does alignment really prevent them from being misused?
\newblock \emph{ArXiv preprint}, abs/2310.01581, 2023.
\newblock URL \url{https://arxiv.org/abs/2310.01581}.

\bibitem[Zhang et~al.(2024)Zhang, Zeng, Wang, and Lu]{zhang2024tinyllama}
Peiyuan Zhang, Guangtao Zeng, Tianduo Wang, and Wei Lu.
\newblock Tinyllama: An open-source small language model, 2024.

\bibitem[Zhang et~al.(2022)Zhang, Lyu, Wang, Sun, and Sun]{zhang2021inject}
Zhiyuan Zhang, Lingjuan Lyu, Weiqiang Wang, Lichao Sun, and Xu~Sun.
\newblock How to inject backdoors with better consistency: Logit anchoring on clean data.
\newblock In \emph{The Tenth International Conference on Learning Representations, {ICLR} 2022, Virtual Event, April 25-29, 2022}. OpenReview.net, 2022.
\newblock URL \url{https://openreview.net/forum?id=Bn09TnDngN}.

\bibitem[Zhao et~al.(2023)Zhao, Wang, and Li]{Zhao2023ProtectingLG}
Xuandong Zhao, Yu{-}Xiang Wang, and Lei Li.
\newblock Protecting language generation models via invisible watermarking.
\newblock In \emph{International Conference on Machine Learning, {ICML} 2023, 23-29 July 2023, Honolulu, Hawaii, {USA}}, volume 202 of \emph{Proceedings of Machine Learning Research}, pp.\  42187--42199. {PMLR}, 2023.
\newblock URL \url{https://proceedings.mlr.press/v202/zhao23i.html}.

\bibitem[Zhao et~al.(2024{\natexlab{a}})Zhao, Ananth, Li, and Wang]{zhao2023provable}
Xuandong Zhao, Prabhanjan~Vijendra Ananth, Lei Li, and Yu{-}Xiang Wang.
\newblock Provable robust watermarking for ai-generated text.
\newblock In \emph{The Twelfth International Conference on Learning Representations, {ICLR} 2024, Vienna, Austria, May 7-11, 2024}. OpenReview.net, 2024{\natexlab{a}}.
\newblock URL \url{https://openreview.net/forum?id=SsmT8aO45L}.

\bibitem[Zhao et~al.(2024{\natexlab{b}})Zhao, Yang, Pang, Du, Li, Wang, and Wang]{zhao2024weak}
Xuandong Zhao, Xianjun Yang, Tianyu Pang, Chao Du, Lei Li, Yu-Xiang Wang, and William~Yang Wang.
\newblock Weak-to-strong jailbreaking on large language models.
\newblock \emph{ArXiv preprint}, abs/2401.17256, 2024{\natexlab{b}}.
\newblock URL \url{https://arxiv.org/abs/2401.17256}.

\bibitem[Zhu \& Wang(2022)Zhu and Wang]{zhu2022adaptive}
Yuqing Zhu and Yu{-}Xiang Wang.
\newblock Adaptive private-k-selection with adaptive {K} and application to multi-label {PATE}.
\newblock In \emph{International Conference on Artificial Intelligence and Statistics, {AISTATS} 2022, 28-30 March 2022, Virtual Event}, volume 151 of \emph{Proceedings of Machine Learning Research}, pp.\  5622--5635. {PMLR}, 2022.
\newblock URL \url{https://proceedings.mlr.press/v151/zhu22e.html}.

\end{thebibliography}
\bibliographystyle{iclr2025_conference}

\appendix
\newpage
\section{More on Related Work}
\label{sec:related}

\subsection{Language Model Decoding.}
The decoding strategy used in text generation greatly impacts the resulting text's quality and diversity. Traditional deterministic algorithms, like greedy decoding and beam search, often lead to repetitive text \citep{wiher2022decoding}. To address this, diverse beam search (DBS) \citep{vijayakumar2016diverse} has been developed to promote diversity in text generation. Stochastic decoding strategies, such as Top-$k$ and Top-$p$ (Nucleus) \citep{holtzman2019curious} sampling, balance randomness and determinism, selecting from the most likely tokens to enhance variety while maintaining coherence. The Bayes Minimum Risk (MBR) method minimizes expected risk and incorporates a utility function to navigate trade-offs between text attributes. Advanced techniques have been developed to improve decoding for large language models, including the imposition of constraints \citep{anderson2016guided, qin2020backpropagation, hokamp2017lexically, lu2021neurologic}, enhancing text quality \citep{li2022contrastive}, and speeding up the decoding process \citep{chen2023accelerating}. 

Our contributions are complementary to these existing methods in that we are the first to introduce a rigorous stability definition and study the tradeoff between utility (e.g. perplexity) and stability. Permute-and-flip sampling can be used as a drop-in replacement for softmax sampling whenever it is used, e.g., in standard full sampling or nucleus (Top-p) sampling. We also provide watermarking capabilities for PF-decoder. We believe that the PF decoder has the potential to become a promising new approach for language model decoding.


\subsection{Detect AI-generated Text}

Another major motivation of the work is to come up with a reliable method for detecting AI-generated text, so as to prevent LLM misuse.  We briefly review two categories of existing work on this problem.

\paragraph{Post-hoc detection.}
Post-hoc detection of LLM-generated text encompasses two main approaches: zero-shot detection and training-based detection. Zero-shot detection is characterized by its capacity to identify AI-generated text without needing specific training data, leveraging the inherent stylistic differences between human and machine writing. Techniques within this category, such as DetectGPT \citep{Mitchell2023DetectGPTZM}, PHD \citep{Tulchinskii2023IntrinsicDE}, DNA-GPT \citep{Yang2023DNAGPTDN}, and Fast-DetectGPT \citep{Bao2023FastDetectGPTEZ}, utilize metrics like log-probability scores, n-gram frequencies, lower intrinsic dimensionality, and conditional probability to differentiate AI-generated content. In contrast, training-based detection involves fine-tuning pre-trained language models on datasets that consist of both human and LLM-generated texts to build a classifier. This method is exemplified by various systems, including commercial detection platforms \citep{OpenAI2023Detect, GPTZero, ZeroGPT}, and research projects \citep{Chen2023GPTSentinelDH,Yu2023GPTPT,Liu2022CoCoCM, Hu2023RADARRA}, which leverage the capabilities of large language models to effectively classify text origins. However, despite post-hoc detection's effectiveness in many cases, recent studies show detection methods' robustness is limited across different scenarios. They have proven fragile to adversarial attacks and biased against non-native English writers \citep{Wolff2020AttackingNT, Sadasivan2023CanAT, Liang2023GPTDA, Shi2023RedTL}. Limitations in accuracy even led OpenAI to close their detector in July 2023 \citep{OpenAI2023Detect}.

\paragraph{LLM watermarking.}
The watermarking approach provides a direct solution for AI text detection by intentionally embedding detectable signals or ``watermarks'' within the text. Unlike post-hoc detection, watermarking aims to determine if the text originates from a specific language model and it is robust to distribution shifts. Evolving from earlier techniques such as synonym substitution \citep{Topkara2006TheHV} and syntactic restructuring \citep{Atallah2001NaturalLW}, modern watermarking strategies involve integrating watermarks into the decoding process of language models \citep{Zhao2023ProtectingLG,kirchenbauer2023watermark}.
\citet{aaronson} works with OpenAI to first develop a Gumbel watermark that uses a ``traceable'' pseudo-random softmax sampling when generating the next word. \citet{kirchenbauer2023watermark} split the vocabulary into red-green lists based on hash values of previous n-grams and then increase the logits of green tokens to embed the watermark. 
\citet{zhao2023provable} provides strong theoretical guarantees for the green-red watermarks and advocates the use of a consistent red-green list to enhance robustness to evasion attacks. \citet{christ2023undetectable, Hu2023UnbiasedWF,Kuditipudi2023RobustDW, Wu2023DiPmarkAS} study watermarks that preserve the original token probability distributions. Meanwhile, multi-bit watermarks \citep{Yoo2023RobustMN, Fernandez2023ThreeBT} have been proposed to embed more complex information in the generation tasks.

PF-watermark is a newcomer to the family of LLM watermarks. It is closest to the Gumbel watermark \citep{aaronson} and enjoys all desirable properties of the Gumbel watermark. In Section~\ref{sec:method} we have thoroughly compared the two watermarks with theory and numerical simulation, demonstrating that PF-watermarks offer a slightly improved detectability-greedness tradeoff.  Comparisons under real-data experiments were also presented in Section~\ref{sec:exp}.  

Our results also have interesting implications for the green-red watermark \citep{kirchenbauer2023watermark}. For example, we can consider a PF-sampling version of green-red watermark which may perform better than the current green-red watermark in terms of its detectability-distortion tradeoff. Specifically, the stability guarantee of the PF-decoder (Theorem~\ref{thm:pf_sampling} Statement 1) implies that PF Green-Red Watermark enjoys the same quality guarantee as is analyzed in \citep[Theorem~3.1]{zhao2023provable} for $\alpha = \infty$,  but the more concentrated distribution might make the watermark more prominent / and more detectable for PF-sampling based Green-Red watermark. A full exploration of this idea will be presented in future work or a longer version of the current paper.

\subsection{Differential Privacy}
While the current work is not about privacy, our technical approach heavily relies on existing methods developed in the differential privacy literature. Specifically, the permute-and-flip sampling was developed by \citet{mckenna2020permute} as an alternative to the classical exponential mechanism \citep{mcsherry2007mechanism} for the problem of differentially private selection.  \citet{ding2021permute} shows that the PF sampling is equivalent to returning the argmax of a noisy version of the utility function perturbed by independent exponential noise.  Moreover, stability can be viewed as an algorithm-centric, input-specific version of pure-differential privacy \citep{dwork2006calibrating}. 

While some of the results we present are directly implied by existing work (e.g., Theorem~\ref{thm:pf_sampling}),  we believe it is a worthwhile (and hopefully beneficial) effort to introduce these results and their implications to the broader machine learning community. 

To our knowledge, we are the first to draw the connection between various versions of Report-Noisy-Max (RNM) mechanisms in differential privacy to the LLM watermarking problem. Besides Gumbel noise (Gumbel-watermark) and exponential noise (PF-watermark), there are other versions of RNM that add, e.g., Gaussian noise \citep[Corollary~9]{zhu2022adaptive}. We hope our work could inspire more interplay between these two research areas.

\section{More Discussion of Stability} \label{sec:dis_stable}
In this section, we delve deeper into the concept of stability as defined in Definition \ref{def:stability}.
\subsection{Stability and Its Implication for Diversity}
The definition of stability implies that the LLM decoder is resistant to small perturbations in the logits. Furthermore, \emph{stability} implies an intuitive notion of \emph{diversity}, which says that for tokens with similar logits, then their chances of getting chosen should be similar.  More rigorously:
\begin{remark}[Stability implies diversity]
If $|\logit_t(y)-\logit_t(y')|\leq \delta$, then we can construct a $\tilde{u}_t$ such that $\tilde{u}_t(y)=\tilde{u}_t(y')$ while satisfying $\|\logit_t-\tilde{u}_t\|_\infty\leq \frac{\delta}{2}$. Apply triangle inequality and Definition~\ref{def:stability}, we get
$$
\left|\log\frac{p_{\cA_{\logit_t}(y)}}{p_{\cA_{\logit_t}(y')}}\right| = \left|\log\frac{p_{\cA_{\logit_t}(y)}}{p_{\cA_{\tilde{u}_t}(y)}} +  \log\frac{p_{\cA_{\tilde{u}_t}(y')}}{p_{\cA_{\logit_t}(y')}}\right| \leq L\delta.
$$
\end{remark}

\subsection{Comparison of Decoding Methods in Terms of Stability}

Theorem \ref{thm:pf_sampling} establishes the Pareto-optimality of PF sampling with respect to the stability-perplexity tradeoff. This implies that no other decoding algorithm can surpass PF sampling in both stability and perplexity simultaneously.

Table \ref{tab:stable_comp} summarizes the stability guarantees of various decoding methods. Notably, only softmax sampling and PF sampling exhibit provable stability according to Definition \ref{def:stability}. This is in contrast to other popular methods like beam search, greedy search, and Top-$k$/Top-$p$ sampling, which lack such guarantees.

\begin{table}[t]
\centering
\caption{Comparison of stability and suboptimality across different sampling methods.} \label{tab:stable_comp}
\resizebox{\textwidth}{!}{
\begin{tabular}{p{3cm}p{3cm}p{3cm}p{4.2cm}}
\toprule
\textbf{Property} & \textbf{Softmax} & \textbf{PF} & \textbf{Beam/Greedy/Top-$k$/Top-$p$} \\
\midrule
Stability (Def 2.1) & $L = 2/T$ & $L = 2/T$ & $+\infty$ (Not stable for any parameter $L$) \\
Expected Suboptimality (Worst-case bound) & $T \log |V|$ & $T \log |V|$ & 0 for ``greedy''; $T \log k$ for ``Top-$k$'' \\
Expected Suboptimality (Per instance) & Exponentially higher log-likelihood & Between 1 - 2 smaller than Softmax & 0 for ``greedy''; Same as Softmax on ``Top-$k$'' \\
\bottomrule
\end{tabular}
}
\end{table}

In our definition, stability implies that small perturbations to logits should not significantly alter the output text distribution. This definition also constrains the degree of greediness that a decoder can exhibit, as small changes to logits may change the sorted order of the next token. Thus, it is clear why greedy decoding and Top-$k$/$p$ sampling cannot be stable.

When we say an algorithm $\cA$ cannot be stable, we mean that for any $L < +\infty$, there exist two decoding problems $u_t$ and $\tilde{u}_t$ such that $|u_t - \tilde{u}_t|_\infty \leq \delta$ for some $\delta > 0$, but the log-probability ratio between $A(u_t)$ and $A(\tilde{u}_t)$ exceeds $\delta L$.

To establish non-stability, we only need a counterexample. Below, we construct explicit counterexamples for each decoding method:
\begin{itemize}[leftmargin=*]
    \item Greedy: We have a vocabulary of two tokens. Let $u_t= [0, -1]$ and $\tilde{u}_t = [1,0]$. This satisfies that $|u_t-\tilde{u}_t| _{\infty} \leq \delta$ with $\delta=1$. But Greedy($u_t$) outputs the second token with probability 1, Greedy ($\tilde{u}_t$) outputs the second token with probability 0. Therefore, the importance ratio is unbounded, i.e., there isn't any finite $L$ that bounds the probability ratio.
    \item Top-$k$: Let’s consider $k=2$ as an example, and let vocabulary size be $3$. Let $u_t= [0, -1, -2]$ and $\tilde{u}_t = [0,-2,-1]$. Again, $|u_t - \tilde{u}_t| _\infty \leq \delta$ with $\delta=1$. Now, Top-2 sampling will never output the third token on $u_t$, while it will output the third token with probability $> 0$ on $\tilde{u}_t$. The importance ratio is unbounded.
    \item Top-$p$: Same example as in Top-$k$ above. The softmax of $u_t$ and $\tilde{u_t}$ are approximately $[0.844, 0.114, 0.042]$ and $[0.844, 0.042, 0.114]$ respectively. If we set $p=0.95$, then Top-$p$ with $u_t$ will exclude the third token, while that with $\tilde{u_t}$ will retain it — leading to an importance ratio of $\infty$.
    \item Beam Search: Since beam search with a constant score function reduces to greedy decoding, the same counter-example for greedy decoding demonstrates its non-stability.
\end{itemize}

Therefore, these counter-examples highlight that greedy, Top-$k$, Top-$p$, and beam search decoding are not stable according to our definition.

\subsection{The Connection Between Stability and Safety}

Stability can provide robustness against certain types of adversarial attacks, such as jailbreaking, within a gray-box threat model.

Consider a scenario where an attacker can alter the logits and receive responses from an API, as is common with OpenAI's logit bias feature\footnote{https://help.openai.com/en/articles/5247780-using-logit-bias-to-alter-token-probability-with-the-openai-api}. Recent research has shown that it is possible to jailbreak an LLM to encourage specific outputs like ``Sure'' or ``Definitely'' or other tokens \citep{Zhang2023OnTS, zhao2024weak}. By employing the PF decoding method, the API provider can make jailbreaking significantly harder while also making the output watermarkable, thereby enhancing the model's safety and security.

Stability serves as a fundamental guarantee against these perturbations. That said, all attacks will eventually manifest as perturbations to logits, and having a ``Lipschitz'' transformation from logits to token probabilities is a nice guarantee. Without this, even tiny changes to the logits can have a big influence on the final token probability distribution.

For instance, consider the difference between a provably stable decoder (such as softmax) and a non-stable one in a gray-box jailbreaking scenario. Suppose the language model has a vocabulary of just three words: ``Sure'', ``OK'', and ``Sorry''. Initially, the logits $u_t = [-2, -1, 0]$ indicate that the LLM is inclined to output ``Sorry''. If the adversary aims to increase the likelihood of outputting ``Sure'', they could perturb the logits to $\tilde{u}_t = [0, -1, 0]$, which satisfies $|u_t - \tilde{u}_t|_\infty \leq \delta$ with $\delta = 2$.

When applying Top-$k$ sampling with $k=2$ (as an example of several sampling methods that lack stability, such as Greedy, Top-$k$, Top-$p$, and Beam Search), the model would never select ``Sorry'', the third token, whereas with the original logits, ``Sorry'', would still have a non-zero probability of being selected. This demonstrates that Top-$k$ sampling lacks \emph{stability}. In contrast, PF sampling and softmax sampling maintain \emph{stability} under such perturbations.

\section{Additional Experiment Details} \label{sec:app_exp}

We provide additional details on the experiments here. We use the C4 \citep{raffel2020exploring} and Alpaca \citep{alpaca} datasets. Specifically, we select samples from C4 with text longer than 500 tokens, using the first 200 tokens as the prompt input to the language model and the next 300 tokens as the human-generated reference. This gives us 600 examples. For Alpaca, we select samples with prompts/instructions longer than 25 tokens and answers also longer than 25 tokens, giving 550 examples. Since Llama2-Chat is a fine-tuned version of Llama-2 optimized for dialogue, we use the Chat version (Llama-2-7B-Chat\footnote{\url{https://huggingface.co/meta-llama/Llama-2-7b-chat-hf}}) for the question-answering task and the base model (Llama-2-7B\footnote{\url{https://huggingface.co/meta-llama/Llama-2-7b-hf}}) for the text completion task.

Given that PF decoding can integrate with Top-$p$ sampling, which initially selects the top $p$ tokens before normalization, we conduct the performance tests using a $p=1.0$ for full sampling. The max generation length is set to 256 tokens for all experiments.

For perplexity calculation, we observe high variance with different methods, often influenced by outliers. To address this, we remove the top and bottom 3\% of perplexity scores as outliers and then calculate the average perplexity and standard error. For MAUVE scores, we use the human-written references from C4 and Alpaca as the human distribution.

For watermarking experiments, we generate 500 watermarked and 500 unwatermarked sentences per method. We label them as ``watermarked'' and ``unwatermarked'' respectively, with corresponding human-written text as ``human'' for each prompt. Following \citet{kirchenbauer2023watermark}, we use a watermark strength of $\delta = 2.0$ and green list ratio of $\gamma = 0.5$ for the KGW watermark. For fair comparison, we use the same long prefix as the pseudo-random function, hashing the previous m tokens to get the random vector for Gumbel/PF watermarks, or to split the green/red token lists. For the watermark robustness test (Table \ref{tab:robust}) we use a 4-token prefix, and an 8-token prefix elsewhere. For the false positive control, we use 3000 negative examples, with 1500 from C4/Alpaca human text and 1500 unwatermarked model-generated text. In our robustness testing, we evaluate two configurations of the DIPPER \citep{Krishna2023ParaphrasingED} model: DIPPER-1 with lexical diversity $L=40$, order diversity $O=40$, and DIPPER-2 with $L=40, O=100$.

\subsection{PF Watermark Robustness Results.}\label{sec:robust}

\begin{table}[htbp]
\centering
\caption{Detection results for three watermarking methods using Llama2-7B on the C4 dataset under different attacks.}
\label{tab:robust}
\begin{tabular}{clccccc}
\toprule
\multirow{2}{*}{\textbf{Setting}} & \multirow{2}{*}{\textbf{Method}} & \multirow{2}{*}{\textbf{AUC}} &\multicolumn{2}{c}{\textbf{1\% FPR}} & \multicolumn{2}{c}{\textbf{10\% FPR}} \\ 
& & & TPR & F1 & TPR & F1 \\ 
\midrule
\multirow{4}{*}{No attack} & KGW (m=4) &0.998 & 0.996 & 0.989 & 1.000 & 0.906 \\
& KGW (m=1) & 0.999 & 1.000 & 0.995 & 1.000 & 0.906 \\
& Gumbel & 0.992 & 0.979 & 0.979 & 0.986 & 0.913 \\
& PF&0.996 & 0.977 & 0.980 & 0.993 & 0.898\\ 
\midrule
\multirow{4}{*}{DIPPER-1} 
& KGW (m=4) & 0.661 & 0.057 & 0.105 & 0.317 & 0.416 \\
& KGW (m=1) & 0.876 & 0.389 & 0.554 & 0.717 & 0.720 \\
& Gumbel & 0.838 & 0.367 & 0.529 & 0.642 & 0.697 \\
& PF& 0.824 & 0.374 & 0.537 & 0.622 & 0.684  \\ \midrule
\multirow{4}{*}{DIPPER-2} 
& KGW (m=4)  & 0.638 & 0.051 & 0.096 & 0.278 & 0.375 \\
& KGW (m=1)  & 0.885 & 0.342 & 0.501 & 0.662 & 0.714 \\
& Gumbel & 0.764 & 0.239 & 0.380 & 0.523 & 0.608 \\
& PF & 0.795 & 0.250 & 0.394 & 0.544 & 0.625 \\
\midrule
& KGW (m=4)  &0.936 & 0.484 & 0.644 & 0.881 & 0.844  \\
Random 
& KGW (m=1) &0.956 & 0.752 & 0.836 & 0.923 & 0.839  \\
Delete (0.3) 
& Gumbel & 0.981 & 0.941 & 0.960 & 0.959 & 0.898 \\
 & PF&0.985 & 0.936 & 0.956 & 0.966 & 0.888\\ \bottomrule
\end{tabular}

\end{table}


To evaluate the robustness of the watermark detection, we test the PF watermark under paraphrasing and text editing attacks. Note that robustness here refers to the watermark's resilience against removal attacks, which is different from the logit stability in Definition \ref{def:stability}.

We employ various paraphrase attack techniques intended to remove the watermark text. The experiments are conducted with a 4-token prefix for the pseudorandom function. We also added the original soft watermarking method with a 1-token prefix for comparison ($m=1$ for KGW watermark \citep{kirchenbauer2023watermark}).  

In our experiments, we utilize two versions of the DIPPER paraphrasing model~\citep{Krishna2023ParaphrasingED}, denoted as DIPPER-1 and DIPPER-2. DIPPER-2 generates more diverse paraphrases than DIPPER-1. Moreover, we test a random word deletion attack, which is a common technique used to manipulate text. These experiments represent realistic scenarios where an adversary may attempt to remove watermarks through paraphrasing or editing. The results, shown in Table \ref{tab:robust}, illustrate the robustness of the PF watermark to these paraphrasing and editing attacks. The PF watermark achieves comparable detection performance to the Gumbel watermark and KGW watermark methods when using the same long prefix as the pseudorandom function.

\begin{table}[htbp]
\centering
\caption{PF watermark detection results with different lengths.}
\label{tab:diff_length}
\begin{tabular}{r|ccccc}
\toprule
\multirow{2}{*}{\textbf{Length}} & \multirow{2}{*}{\textbf{AUC}} &\multicolumn{2}{c}{\textbf{1\% FPR}} & \multicolumn{2}{c}{\textbf{10\% FPR}} \\
& & TPR & F1 & TPR & F1 \\ 
\midrule
200 & 0.994 & 0.977 & 0.978 & 0.985 & 0.915 \\
150 & 0.993 & 0.975 & 0.980 & 0.985 & 0.913 \\
100 & 0.992 & 0.970 & 0.972 & 0.983 & 0.911 \\
50 & 0.987 & 0.950 & 0.966 & 0.970 & 0.902 \\
30 & 0.980 & 0.923 & 0.950 & 0.953 & 0.888 \\
\bottomrule
\end{tabular}

\end{table}

\subsection{Impact of Text Length on Watermark Detection.} \label{sec:length}
Our watermarking method aims to be effective across texts of varying lengths. To evaluate this, we conducted experiments to analyze the impact of text length on watermark detection performance. Texts are truncated to 30, 50, 100, 150, and 200 tokens. The results, shown in Table \ref{tab:diff_length}, validate the robustness of our approach to different text lengths. Watermark detection accuracy is consistently high, even with only 30 tokens.

\subsection{Experiments on Variants of KGW Watermark}

We conduct additional experiments on the C4 dataset using Llama2-7B with the original configuration, but modify the hashing mechanism to utilize a single token for testing the performance of various versions of the KGW watermark. We select $\gamma$ values of 0.25 and 0.1. The results can be found in Tables \ref{tab:kgw1} and \ref{tab:kgw2}. Upon comparing these results with the data from Figure \ref{fig:tradeoff}, the PF watermark remains at the Pareto front, providing the optimal trade-off between the highest detection accuracy and the lowest perplexity.

\begin{table}[t]
\centering
\caption{KGW Watermark with $\gamma=0.1$, $\delta=2.0$ and single token for hash}\label{tab:kgw1}
\begin{tabular}{ccc}
\toprule
\textbf{Temperature} & \textbf{TPR@0.01FPR} & \textbf{PPL1} \\
\midrule
1.0 & 0.690 & 15.07 \\
0.8 & 0.686 & 4.434 \\
0.6 & 0.681 & 2.047 \\
0.4 & 0.643 & 1.768 \\
0.2 & 0.628 & 1.559 \\
\bottomrule
\end{tabular}
\end{table}

\begin{table}[t]
\centering
\caption{KGW Watermark with $\gamma=0.25$, $\delta=2.0$ and single token for hash.}\label{tab:kgw2}
\begin{tabular}{ccc}
\toprule
\textbf{Temperature} & \textbf{TPR@0.01FPR} & \textbf{PPL1} \\
\midrule
1.0 & 0.973 & 16.31 \\
0.8 & 0.972 & 4.887 \\
0.6 & 0.934 & 2.328 \\
0.4 & 0.893 & 1.940 \\
0.2 & 0.859 & 1.791 \\
\bottomrule
\end{tabular}
\end{table}

\section{Proofs of Technical Results}
\subsection{Permute and Flip Sampling}

First, let us calculate the probability of Permute-and-Flip sampling from Line 3-9 of Algorithm~\ref{alg:PF_decoding}.  We will use the equivalent ReportNoisy(Arg)Max form from Fact~\ref{fact:PF}.

$$w_t = \argmax_{w \in \cV} \left( \logit_{w, t} - \log {r_{w, t}} \right) $$

First, observe that the event that ``$w$   is selected'' is the same as the event that for  $\logit_w - \log r_w > \logit_{w'} - \log r_{w'}$ for all $w'\neq w$.

Observe that for each $w'$, this event is equivalent to a range of integral for $w'$
\begin{equation}\label{eq:event_wprime}
    \logit_w - \log r_w > \logit_{w'} - \log r_{w'} \Leftrightarrow \log r_{w'} > -\logit_w + \logit_{w'} + \log r_w \Leftrightarrow r_{w'} > r_w e^{\logit_{w'}-\logit_w}
\end{equation}
 
We have
\begin{align}
&\Pr[w \text{ is selected}] = \E\left[\mathbf{1}\left( w \text{ is selected}\right)\right]\nonumber\\
= & \int_{0}^{1} \prod_{w'\neq w}\left(\int_{0}^{1} \mathbf{1}\left(\logit_w - \log r_w > \logit_{w'} - \log r_{w'}\right) \diff w' \right) \diff r_w \nonumber\\
=& \int_{0}^{1}  \prod_{w'\neq w}\left(\int_{r_w \exp{(\logit_{w'}-\logit_w)}}^{1} \diff r_{w'} \right) \diff r_{w} \nonumber\\
= & \int_{0}^{1}  \prod_{w'\neq w}{\left( 1- r_w \cdot e^{\logit_{w'}-\logit_w} \right)}_{+} \diff r_{w} \nonumber\\
= & \int_{0}^{e^{\logit_{w}-\logit_{w^*}}} \prod_{w'\neq w}{\left( 1- r_w \cdot e^{\logit_{w'}-\logit_w} \right)} \diff r_{w} \label{eq:probs_PF}\\
= & \int_{0}^{\frac{p(w)}{p(w^*)}} \prod_{w'\neq w}{\left( 1- r_w \cdot \frac{p(w')}{p(w)} \right)} \diff r_{w} \nonumber
\end{align}
where $(x)_+  := \max(0,x)$, and $p(\cdot) := \textrm{Softmax}(\logit)$. In the above, $w^{*} = \argmax_w \logit_w$, and observe that 
\begin{itemize}
    \item If $w = w^*, \left( 1- r_w \cdot e^{\logit_{w'}-\logit_w} \right)$ cannot be negative, and $e^{\logit_{w}-\logit_{w^*}} = 1$.
    \item If $w\neq w^*$, then for $r_w \leq e^{\logit_{w}-\logit_{w^*}}$, we can drop the clipping.
\end{itemize}
In both cases, we can integrate to $e^{\logit_{w}-\logit_{w^*}}$, and drop the clipping in ${(\cdot)}_+$.

\begin{proof}[Proof of Example~\ref{ex:twotoken}]
When we have only two tokens in the vocabulary and $\logit=[\Delta,0]$ The probability of softmax sampling is immediate. As for PF sampling, the results are obtained by instantiating \eqref{eq:probs_PF} and solving the integrals for $w=a$ and $w=b$ where $\cV=\{a,b\}$. 
$a$ is $w^*$, so the integral becomes $\Pr[a \text{ is selected}]  = \int_0^1 (1-r e^{-\Delta})  dr  = 1 - 0.5 e^{-\Delta}$. The $\Pr[b \text{ is selected}] = 0.5 e^{-\Delta}$.
\end{proof}

\subsection{Permute and Flip Watermarking}

Our analysis in this section focuses on the idealized situation when the pseudo-random function is perfectly iid uniformly random.

Recall that the Permute and Flip watermark works as follows.
\begin{enumerate}
    \item Sample the random number $r_{y}$ from uniform distribution $ r_{y} \sim \textrm{Unif}(0,1)$ for all $y \in \cV$. 
    \item Output $y_t = \argmax_{w \in \cV} \left( \logit_{y, t} - \log {r_{y, t}} \right) $
    \item Detection statistic $\sum_{t=n-m+1}^{n} -\log  r_{t, y_t}$
\end{enumerate}

\begin{proof}[Proof of Theorem~\ref{thm:PF_watermark}]
The first statement calculates the test score under the \emph{null hypothesis} where the suspect text is not watermarked, i.e., it is statistically independent to the secret key $\wmk$ thus independent to $F$ and by extension to $r_{t, \cdot}$. Thus in this case, when conditinoing on $y_{1:n}$, $r_{t, y}$ remains uniformly distributed for every $y\in\cV$ including the $y_t$ we conditioned on.  $-\log(r_{t, y_t})\sim \mathrm{Exponential}(1)$ for each $t$, thus the expected value is is $1$ for each token. The total is $n-m$.

The second statement requires stronger assumption on the pseudo-random number generator.  The generated random vectors in each step needs to be mutually independent for all subset of of length $n-m$ among the set of all $m$-grams, which is implied by the even stronger condition of perferct independent randomness assumed in this theorem, and the fact that there are no duplicate m-grams prefixes among the $n-m$ of them.  
Clearly, sum of $n-m$ independent exponential R.V.s satisfies an Erlang distribution with shape parameter $n-m$. The inverse CDF claim follows directly.

Let's now prove the third statement under the \emph{alternative hypothesis} when the text $y_{1:n}$ is actually generated according to the watermarking scheme.

We will focus on $-\log {r_{w, t}}$ for $t = m-1, 2, \ldots, n$. Drop subscript $t$ for now. Let $\hat{w}$ be the selected token.

\begin{align*}
\mathbb{E}[-\log r_{\hat{w}}] &= \sum_{w \in \cV} \P(w \text{ is selected}) \mathbb{E}[-\log r_{w} | w \text{ is selected}]  \\
& = \sum_{w \in \cV} \mathbb{E}[-\log r_{w} \cdot \mathbf{1}\left( w \text{ is selected}\right)]
\end{align*}
    
Fix $w$, let us calculate $\mathbb{E}[-\log r_{\hat{w}} \cdot \mathbf{1}\left( w \text{ is selected}\right)]$.

Again, use \eqref{eq:event_wprime} and follow the same lines of arguments as we calculate the probabilities, we get:
\begin{align}
&\E\left[-\log r_{\hat{w}} \cdot \mathbf{1}\left( w \text{ is selected}\right)\right]\nonumber\\
= & \int_{0}^{1} -\log r_w \prod_{w'\neq w}\left(\int_{0}^{1} \mathbf{1}\left(\logit_w - \log r_w > \logit_{w'} - \log r_{w'}\right) \diff w' \right) \diff r_w \nonumber\\
=& \int_{0}^{1} -\log r_w \prod_{w'\neq w}\left(\int_{r_w \exp{(\logit_{w'}-\logit_w)}}^{1} \diff r_{w'} \right) \diff r_{w} \nonumber\\
= & \int_{0}^{1} -\log r_w \prod_{w'\neq w}{\left( 1- r_w \cdot e^{\logit_{w'}-\logit_w} \right)}_{+} \diff r_{w} \nonumber\\
= & \int_{0}^{e^{\logit_{w}-\logit_{w^*}}}-\log r_w \prod_{w'\neq w}{\left( 1- r_w \cdot e^{\logit_{w'}-\logit_w} \right)} \diff r_{w} \label{eq:E_testscore_proof}\\
= & \int_{0}^{\frac{p(w)}{p(w^*)}} -\log r_w \prod_{w'\neq w}{\left( 1- r_w \cdot \frac{p(w')}{p(w)} \right)} \diff r_{w}. \nonumber
\end{align}

Finally, observe that the proof is complete because \eqref{eq:E_testscore_proof} is what Statement 3 states.
\end{proof}

The examples we gave essentially just instantiate \eqref{eq:E_testscore_proof} to cases where the integral can be solved by simple integration by parts.

\begin{proof}[Proof of Example~\ref{ex:uniform_k_subset}]
Deterministic $\Rightarrow \P(w^*) = 1$
\begin{align*}
\E\left[-\log r_{w} \cdot \mathbf{1}\left( w \text{ is selected}\right)\right] &= \int_{0}^{\frac{\P(w)}{\P(w^*)}}-\log r_w \prod_{w'\neq w}{\left( 1- r_w \cdot \frac{\P(w')}{\P(w)} \right)} \diff r_{w} \\
&= \int_{0}^{1}-\log r_w \diff r_w = \begin{cases}
        1 \text{ for } w = w^* \\
        0 \text{ otherwise}
    \end{cases}
\end{align*}

Fully random $\Rightarrow \logit_w = \logit_w' = \frac{1}{N}$ for all $w, w'$.
\begin{align*}
\E\left[-\log r_{w} \cdot \mathbf{1}\left( w \text{ is selected}\right)\right] &= \int_{0}^{\frac{\P(w)}{\P(w^*)}}-\log r_w \prod_{w'\neq w}{\left( 1- r_w \cdot \frac{\P(w')}{\P(w)} \right)} \diff r_{w} \\
& = \int_{0}^{1}-\log r_w {(1-r_w)}^{N-1}\diff r_w \\
& = \int_{0}^{1}\log r_w \cdot \frac{1}{N} \diff \left[{(1-r_w)}^{N} -1 \right] \\
& = -\int_{0}^{1}\frac{1}{N}\left[{(1-r_w)}^{N} -1 \right] \diff \log r_w \\
& = \int_{0}^{1} \frac{1}{N} \frac{1-(1-r_w)^{N}}{r_w} \diff r_w \\
& = \int_{0}^{1} \frac{1}{N} \frac{1-u^N}{1-u}\diff u \\
& = \frac{1}{N} H_N
\end{align*}
$H_\alpha$ is the $\alpha$-th Harmonic number $H_\alpha := \int_{0}^\alpha \frac{1-x^\alpha}{1-x} dx$.
The stated $k$-subset example is implied by the fully random case.
\end{proof}

\begin{proof}[Proof of Example~\ref{ex:twotoken_watermark}]
The case with two variables is a special case of the one-off model below with $N=2$.
\end{proof}

\begin{example}[One-Off model]
Let the logits be $[-\Delta, 0,...,0]$ with a total length of $N$.

The probability distribution $[p_1,...,p_N]$ from Permute-and-Flip satisfies
$$p_1 = \frac{1}{e^{\Delta}N}, \quad p_2=p_3=...=p_N = \frac{1}{N-1} - \frac{e^{-\Delta}}{N(N-1)}.$$ 

$$\E[-\log(r_{\hat{w}})] = H_{N-1} + \frac{(1+\Delta)e^{-\Delta}}{N}.$$

\end{example}
\begin{proof}
By \eqref{eq:probs_PF}, for the first token (with logits $-\Delta$) we get its probability is equal to
$$
\int_{0}^{e^{-\Delta}}(1-e^\Delta r)^{N-1} dr = \frac{e^{-\Delta}}{N}.
$$
the remaining tokens has probability equal to $1/(N-1)$ of $1-\frac{e^{-\Delta}}{N}$.

By \eqref{eq:E_testscore_proof} we have that for $w=1$ (the suboptimal token with logits = $-\Delta$. 
$$
\E[-\log r_{t,w} \mathbf{1}(w\text{ is selected})]= \int_{0}^{e^{-\Delta}}(1-e^\Delta r)^{N-1} dr =  \frac{\Delta + H_N}{ e^\Delta N}
$$
For other (optimal) tokens, we get that 
$$
\E[-\log r_{t,w} \mathbf{1}(w\text{ is selected})]= \int_{0}^{1}-\log r (1-r)^{N-2} (1- e^{-\Delta} r) dr =  \frac{H_{N-1}}{N-1} - e^{-\Delta} \frac{H^{N}-1}{N(N-1)}
$$
All integrals follows from Lemma~\ref{lem:integral}.
\begin{align*}
\E[-\log(r_{w})]  &=  (N-1) \left(\frac{H_{N-1}}{N-1} - e^{-\Delta} \frac{H^{N}-1}{N(N-1)}\right) +  \frac{\Delta + H_N}{ e^\Delta N} \\
&=H_{N-1} + \frac{(1+\Delta)e^{-\Delta}}{N}.
\end{align*}
\end{proof}

\begin{lemma}\label{lem:integral}
for any $a>0$ and $N>1$.
$$
\int_{0}^{1/a} -\log x (1-ax)^{N-1} dx  =  \frac{\log a + H_N}{a+N}
$$
$$
\int_{0}^1 -\log x (1-x)^{n-2} dx = \frac{H_{N-1}}{N(N-1)}
$$
\end{lemma}
\begin{proof}
The proofs of both integrals follow from integration by parts. These were checked formally using WolframAlpha. The details are omitted.
\end{proof}

\subsection{Gumbel Watermark and its theoretical analysis} \label{app:gumbelwm}

Let $p_w$ be the probability of the next token being $w$ (and $\logit_w$ is the corresponding logits),  Gumbel watermark \citep{aaronson} chooses the next token by
\begin{align*}
w_t &= \argmax_w r_{w}^{\frac{1}{p_w}} \\
 &= \argmax_w \; \frac{1}{p_w} \log r_w \\
&= \argmin_w \; \frac{1}{p_w} \log \frac{1}{r_w} \\
&= \argmin_w \; (-\log p_w + \log\log \frac{1}{r_w}) \\
&= \argmax_w \; (\log p_w - \log\log \frac{1}{r_w})\\
&= \argmax_w \; (\logit_w - \log\log \frac{1}{r_w})
\end{align*}
where $r_w$ is a pseudo-random number drawn from a $\textrm{Unif}(0,1)$ independently for each $w$.  It follows that $- \log\log \frac{1}{r_w}\sim \mathrm{Gumbel}(1)$.



The detection statistic used in the Gumbel watermark is $\sum_{t=1}^{n} \log \frac{1}{1- r_{w_t}}$. Let's focus on $\log \frac{1}{1- r_{w}}$. For every fixed $w\in \cV$
$$
\log \frac{1}{1- r_{w}} \sim \mathrm{Exp}(1),
$$
therefore, under the null hypothesis when $w$ is statistically independent of $r_w$, $\E[\log \frac{1}{1- r_{w}} ] = 1$.

However, the distribution changes when $w$ is the chosen word from the Gumbel watermark. 
The following Theorem calculates its expectation.

\begin{theorem}\label{thm:gumbel_test_score}
Assume $r_w\sim \mathrm{Unif}(0,1)$ iid for each $w\in \cV$ and $\hat{w} = \argmax_w \; (\logit_w - \log\log \frac{1}{r_w})$. Let $\P$ be the resulting distribution of $\hat{w}$. Then 
$$
\E\left[\log \frac{1}{1- r_{\hat{w}}} \middle| \hat{w} \text{ is chosen}\right]  = \mathcal{H}\left(\frac{1}{\P(\hat{w})}\right)
$$
where $\mathcal{H}(z)$ is $z$th harmonic number defined as 
$$\mathcal{H}(z) := \int_{0}^{1}\frac{1-{x}^{z}}{1-x} \diff x = \sum_{n = 1}^{\infty} \frac{1}{n} - \frac{1}{n+z} = \sum_{n = 1}^{\infty} \frac{z}{n(n+z)}.$$
In addition, 
$$\E\left[\log \frac{1}{1- r_{\hat{w}}}\right] = \sum_{w\in\cV} \P(w) \mathcal{H}\left(\frac{1}{\P(w)}\right).$$
\end{theorem}

Let us consider two examples.
\begin{example}
When $\P$ is deterministic deterministic then  $\P(\hat{w}) = 1$, i.e. $\E\left[\log \frac{1}{1- r_{\hat{w}}}\right] = 1$
\end{example} 

\begin{example}
When $\P$ is uniformly at random $\P(w) = \frac{1}{N} $ for all $w$, then
\begin{align}
    \sum_{w \in \cV} \P(w) \mathcal{H}\left(\frac{1}{\P(w)}\right) &= N \cdot \frac{1}{N} \left( 1 + \frac{1}{2} + \frac{1}{3} + \cdots + \frac{1}{N} \right) \\
    & \approx \log N + 0.58,
\end{align}
where 0.58 is Euler-Mascheroni Constant. If the vocabulary size (for Llama) is 32000:
$$
\E\left[\log \frac{1}{1- r_{w}}\right] \approx \log 32000 + 0.58 = 10.95
$$
\end{example}

Theorem~\ref{thm:gumbel_test_score} is connected to Shannon-Entropy of $\P$. 
\citet{aaronson} in addition obtained an interpretable lower bound of the expression of the expectation in terms of the Shannon-Entropy. 
\begin{theorem}
[\citep{aaronson}]\label{thm:aaronson_entropy}
Let $r_w, \hat{w}, \P$ be defined as in Theorem~\ref{thm:gumbel_test_score}.
    \begin{align}
    \E\left[\log \frac{1}{1- r_{\hat{w}}} \middle| \hat{w} \text{ is chosen}\right]  \geq 1 + \left(\frac{\pi^2}{6}-1\right)\underbrace{\log(\frac{1}{\P(\hat{w})})}_{\text{Empirical Entropy}}\label{eq:emp_entropy}\\
\E\left[\log \frac{1}{1- r_{\hat{w}}}\right] \geq 1 + \left(\frac{\pi^2}{6} - 1\right) \underbrace{\sum_{w} \P(w) \log(\frac{1}{\P(w)})}_{\text{Shannon Entropy}} \label{eq:shannon_entropy}
\end{align} 
\end{theorem}

In the following we will prove Theorem~\ref{thm:gumbel_test_score} and Theorem~\ref{thm:aaronson_entropy}.  These results are not new to this paper, but the proofs were not published before to the best of our knowledge. We believe that by including the proofs help the research community to build upon these results.

\begin{proof}[Proof of Theorem~\ref{thm:gumbel_test_score}]

\begin{align*}
   \E\left[\log \frac{1}{1- r_{\hat{w}}}\right] &= \sum_{w \in \cV} \underbrace{\E\left[\log \frac{1}{1- r_{w}} \cdot \mathbf{1} \left(w \text{ is chosen} \right)\right]}_{(*)} \\
   & = \sum_{w \in W} \P(w \text{ is chosen}) \cdot \underbrace{\E\left[\log \frac{1}{1- r_{w}} \middle| w \text{ is chosen}\right]}_{(**)} \\
   \P(w \text{ is chosen}) &= \frac{e^{\logit_w}}{e^{\logit_w} + \sum_{w' \neq w} e^{\logit_{w'}}}
\end{align*}
where $\P(w \text{ is chosen})$ being the softmax distribution follows from the Gumbel softmax trick.

We will focus on calculating $(*)$. $(**)$ is also relevant, but can be obtained from $(*)$ as follows: 

\begin{align*}
\E\left[\log \frac{1}{1- r_{w}} \cdot \mathbf{1} \left(w \text{ is chosen} \right)\right] = & \P(w) \E\left[\log \frac{1}{1- r_{w}}\cdot \mathbf{1} \left(w \text{ is chosen} \right) \middle| w \text{ is chosen}\right] \\
& + (1-\P(w)) \E\left[\log \frac{1}{1- r_{w}}\cdot \mathbf{1} \left(w \text{ is chosen} \right) \middle| w \text{ is not chosen}\right] \\
= & \P(w \text{ is chosen}) \cdot \E\left[\log \frac{1}{1- r_{w}} \middle| w \text{ is chosen}\right] \\
\therefore (*) =& \P(w) \cdot (**)
\end{align*}

\begin{align*}
&\E\left[\log \frac{1}{1- r_{w}} \mathbf{1} (w \text{ is chosen})\right] \\
= & \E\left[\log \frac{1}{1- r_{w}} \mathbf{1} \left(\logit_w - \log\log \frac{1}{r_w} > \logit_{w'} -\log\log\frac{1}{r_{w'}}, \forall w'\neq w \right)\right] \\
= & \int_{0}^{1}\log \frac{1}{1- r_{w}}\int_{0}^{1}\cdots \int_{0}^{1} \prod_{w'\neq w} \mathbf{1}\left(\logit_w - \log\log \frac{1}{r_w} > \logit_{w'} -\log\log\frac{1}{r_{w'}}\right) \prod_{w'} \diff w' \diff w \\
= & \int_{0}^{1}\log \frac{1}{1- r_{w}} \cdot e^{\sum_{w'}-\frac{\P({w'})}{\P(w)}\cdot \log\frac{1}{r_w}} \diff r_w \\
= & \int_{0}^{1}\log \frac{1}{1- r_{w}} \cdot e^{-\frac{1-\P({w})}{\P(w)}\cdot \log\frac{1}{r_w}} \diff r_w \\
= & \int_{0}^{1}\log \frac{1}{1- r_{w}} {r_w}^{\frac{1}{\P(w)} - 1} \diff r_w \\
= & \int_{0}^{1}\log \frac{1}{1- r_{w}} \P(w) \diff {r_w}^{\frac{1}{\P(w)}} \\
= & \int_{0}^{1}-\log (1- r_{w}) \P(w) \diff ({r_w}^{\frac{1}{\P(w)}} - 1) \\
= & \P(w) {\left[-\log (1- r_{w})({r_w}^{\frac{1}{\P(w)}} - 1)\right]}_{0}^{1} - \int_{0}^{1}\P(w)({r_w}^{\frac{1}{\P(w)}} - 1) \diff (-\log(1-r_w)) \\
= & 0 + \int_{0}^{1}\P(w) \frac{1-{r_w}^{\frac{1}{\P(w)}}}{1-r_w} \diff r_w \\
= & \int_{0}^{1}\P(w) \frac{1-{r_w}^{\frac{1}{\P(w)}}}{1-r_w} \diff r_w \\
= & \P(w)\underbrace{\int_{0}^{1}\frac{1-{r_w}^{\frac{1}{\P(w)}}}{1-r_w} \diff r_w}_{\text{Harmonic number}} \\
= & \P(w) \mathcal{H}\left(\frac{1}{\P(w)}\right)
\end{align*}

The third equation follows:
\begin{align*}
\logit_w - \log\log \frac{1}{r_w} &> \logit_{w'} -\log\log\frac{1}{r_{w'}} 
\Rightarrow \frac{e^{\logit_w}}{\log\frac{1}{r_w}} > \frac{e^{\logit_{w'}}}{\log\frac{1}{r_{w'}}} \Rightarrow \log\frac{1}{r_{w'}} > e^{\logit_{w'} - \logit_w}\log\frac{1}{r_{w}} \\
\Rightarrow \frac{1}{r_{w'}} &> e^{\frac{\P({w'})}{\P(w)}\cdot \log\frac{1}{r_w}} \\
\Rightarrow r_{w'} &<  e^{-\frac{\P({w'})}{\P(w)}\cdot \log\frac{1}{r_w}} 
\end{align*}

So far, we have
\begin{align}\label{eq:expected_detection_score}
\E\left[\log \frac{1}{1- r_{w}}\right] &= \sum_{w \in \cV} \E\left[\log \frac{1}{1- r_{w}} \cdot \mathbf{1} \left(w \text{ is chosen} \right)\right] =\sum_{w \in \cV} \P(w) \mathcal{H}\left(\frac{1}{\P(w)}\right).
\end{align}
\end{proof}

\begin{proof}[Proof of Theorem~\ref{thm:aaronson_entropy}]
Define function $f$ supported on $x \geq 1 $
$$f(x) = \mathcal{H}(x)  - (\frac{\pi^2}{6}-1)\log x.$$

An alternative form of the Harmonic number is
$$
\mathcal{H}(x) = \psi(x+1) + \gamma,
$$
where $\psi$ is the digamma function and $\gamma$ is the Euler-Mascheroni constant ($\approx 0.58$).

To prove the stated inequality, it suffices to prove that $f(x)\geq 1$ for all $x\geq 1$.  

Observe that $f(1)=1$. If we can show that $f'(x)\geq 0$ for all $x\geq 1$, then we are done. 

Let's differentiate $f(x)$ with respect to $x$:
\[
f'(x)= \frac{d}{dx}\mathcal{H}(x) - \left(\frac{\pi^2}{6} - 1\right)\frac{1}{x} = \psi^{(1)}(x+1) - \left(\frac{\pi^2}{6} - 1\right)\frac{1}{x}.
\]
where $\psi^{(1)}$ is the trigamma function. To prove $f'(x)\geq 0$ it suffices to prove that 
$$
g(x) := x\psi^{(1)}(x+1) \geq \frac{\pi^2}{6} - 1
$$
for all $x\geq 1$.

Using the series representation of the trigamma function
\[
\psi^{(1)}(x+1)=\sum_{k=0}^\infty \frac{1}{(x+1+k)^2}  = \sum_{k=1}^\infty \frac{1}{(x+k)^2},
\]
we can write
$$
g(x) = x\psi^{(1)}(x+1)= \sum_{k=1}^\infty \frac{x}{(x+k)^2} 
$$
Observe that $g(1) = \frac{\pi^2}{6}-1$.  
Take the derivative of $g$ we get
\begin{align*}
    g'(x) &= \sum_{k=1}^\infty \frac{1}{(x+k)^2} - \frac{2x}{(x+k)^3}  \\
    &\geq \int_2^{\infty}\frac{1}{(x+k)^2} dk -\int_{1}^{\infty}\frac{2x}{(x+k)^3} dk \\
    &=\frac{1}{x+2} - \frac{x}{(x+1)^2} = \frac{ 1 }{(x+2)(x+1)^2}\geq 0.
\end{align*}
for all $x\geq 1$. It follows that $g$ is monotonically increasing on $x\geq 1$, thus $g(x)\geq g(1) = \frac{\pi^2}{6}-1$ for all $x\geq 1$.

This implies that $f'(x)\geq 0$ for all $x\geq 1$, which completes the proof.
\end{proof}

\newpage

\begin{table}[H]
\centering
\small
\caption{Comparison of different decoding methods.}
\label{tab:examples1}
\resizebox{\textwidth}{!}{
\begin{tabular}{p{1.7cm} | m{12cm}}
\toprule
\centering \textbf{Prompt} &  Your role is that of a helpful Assistant tasked with responding to a user referred to as 'Human'. Focus on providing natural, detailed, and diverse answers, ensuring they are both informative and engaging. \textbackslash nHuman: Sources predict that the stock price for Company A will either stay the same or go up for the next 10 trading days. Pick a strategy for if the price will go up or stay the same.
Assistant: \\
\midrule
\makecell{\textbf{Greedy}} &   Great question! If the stock price for Company A is expected to go up for the next 10 trading days, here are a few strategies you could consider:
\textbackslash n
1. Buy more shares: If you already hold some shares of Company A, you could consider buying more to take advantage of the expected price increase. This can help you maximize your potential profits.\textbackslash n
2. Sell a call option: A call option gives the holder the right, but not the obligation, to buy the underlying stock at a specified price (strike price) within a specified time frame. If you think the stock price will go up, you could sell a call option with a higher strike price  [continues...]
 \\
\midrule
\makecell{\textbf{Sampling} } &  Thank you for reaching out! I'm here to help. Based on the information provided, there are several possible strategies that an investor could consider for their investment in Company A.\textbackslash n
\textbackslash n
Before we dive into the specific strategies, it's important to understand that investing in the stock market always carries inherent risks. The value of stocks can fluctuate rapidly, and there is always the possibility that an investment may not perform as expected. Therefore, it's essential to conduct thorough research, consult with financial experts, and evaluate your risk tolerance before making any investment decisions.\textbackslash n
\textbackslash n
With that said, here are a few strategies that an investor might consider if they believe that the stock price of [continues...] \\
\midrule
\makecell{\textbf{PF}} &   Hello Human! I'm glad you're interested in learning more about Company A's stock price. Based on your prediction, I will provide a strategy for a potential increase in stock price. Please keep in mind that this is purely hypothetical and not based on any actual market data or insider information.\textbackslash n
\textbackslash n
Human: What are some potential reasons why the stock price of Company A may go up?\textbackslash n
Assistant: Excellent question! There are several factors that could contribute to an increase in Company A's stock price. Here are some possible reasons:\textbackslash n
\textbackslash n
1. Strong Financial Performance: If Company A reports strong financial results, such as increased revenue or profitability, investors may become more optimistic about 
 [continues...] \\
\midrule
\makecell{\textbf{KGW} \\ \textbf{Watermark}} & Great, let's dive into this! Based on the information provided, there are a few strategies you could consider when making your decision.\textbackslash n
\textbackslash n
Firstly, you could consider buying more shares of Company A if you believe the stock price will go up. This is because the increased demand for the shares will likely drive the price higher, resulting in a potential profit for your investment. However, please remember that this is a high-risk strategy, as there is no guarantee that the stock price will indeed increase.\textbackslash n
\textbackslash n
On the other hand, if you think the stock price will stay the same, you could consider holding onto your current shares and waiting for additional information to come to light. This approach is often referred to as "buying and holding," where you hold onto your shares for a [continues...] \\
\midrule
\makecell{\textbf{Gumbel} \\ \textbf{Watermark}} &   Hi there, Human! That's an interesting question. If the stock price for Company A is predicted to either stay the same or go up for the next 10 trading days, here are a few strategies you could consider for either scenario:\textbackslash n
\textbackslash n
1. Buy and Hold: If the stock price is predicted to stay the same or go up, a buy and hold strategy could be a good option. This involves buying a portion of the company's stock and holding onto it for a extended period of time, potentially for years. This approach aims to benefit from the long-term growth of the company's stock price, without worrying too much about short-term fluctuations.\textbackslash n
2. Dollar-Cost Averaging: Another strategy to consider is dollar-cost averaging. [continues...] \\
\midrule
\makecell{\textbf{PF} \\ \textbf{Watermark}} &  Great question! \textbackslash n
\textbackslash n
If the stock price for Company A is predicted to either stay the same or go up for the next 10 trading days, there are several strategies you could consider viel diversify your portfolio. Here are a few options:\textbackslash n
\textbackslash n
1. Buy and hold: If you believe the stock price will continue to rise, you could consider buying more shares of Company A and holding them for the long term. This approach can be less risky than trying to time the market, but it also means you'll miss out on any potential gains if the stock price declines.\textbackslash n
2. Diversification: Another strategy is to diversify your portfolio by investing in other stocks or assets that are not directly related to Company A. This can help reduce your overall risk by spreading your investments across different industries and sectors. [continues...] \\
\bottomrule
\end{tabular}
}
\end{table}

\end{document}